%% file: main.tex
\newif\ifjrssb
\jrssbtrue

\documentclass[a4paper,12pt]{article}
\usepackage{tikz}
\usetikzlibrary{decorations.pathreplacing, calc}

\ifjrssb
  \usepackage[a4paper, margin=1in]{geometry}
\else
  \usepackage[a4paper, top=0.75in, bottom=0.75in, left=1in, right=1in]{geometry}
\fi

\ifjrssb\else
  \fontsize{12pt}{14pt}\selectfont
\fi

\input{preamble}

\ifjrssb  \fi

\ifjrssb
  \newcommand{\bibitemsep}{1.5\baselineskip}
\else
  \newcommand{\bibitemsep}{1.2\baselineskip}
\fi
\patchcmd{\thebibliography}
  {\settowidth}
  {\setlength{\itemsep}{\bibitemsep}\setlength{\parskip}{0pt}\setstretch{1}\settowidth}
  {}{}

\usepackage{placeins}  

\title{Bayesian Invariance Modeling \\ of Multi-Environment Data}

\author{
  Luhuan Wu$^{\S}$, Mingzhang Yin$^{\#}$, Yixin Wang$^{*}$, \\
  John P. Cunningham$^{\dag}$, David M. Blei$^{\dag,\ddag}$
}

\date{}
\begin{document}
\maketitle

\renewcommand{\thefootnote}{}%
\footnotetext{$^{\S}$Department of Applied Mathematics and Statistics, Johns Hopkins University. 
   $^{\dag}$Department of Statistics, Columbia University.
   $^{\ddag}$Department of Computer Science, Columbia University.
    $^{\#}$Warrington College of Business, University of Florida.
    $^{*}$Department of Statistics, University of  Michigan.
   \textit{Address for correspondence:} Luhuan Wu. Email: \texttt{lwu86@jh.edu}
   }%
\addtocounter{footnote}{-1}%
\renewcommand{\thefootnote}{\arabic{footnote}}%

\input{sections/abstract}
\input{sections/introduction}

\input{sections/method}

\input{sections/theory}
\input{sections/experiments}
\input{sections/discussion}

\section*{Data Availability Statement}
This paper analyzes publicly available, third-party data together with synthetic data generated by our code; no new empirical data were collected. The yeast gene-perturbation data are from \citet{kemmeren2014large}, with the updated version used here available at \url{https://deleteome.holstegelab.nl/}. The synthetic data in the simulation studies are fully reproducible from the accompanying code. Code to reproduce all experiments, figures, and tables is publicly available at \url{https://github.com/LuhuanWu/BayesianInvariantPrediction}.

\bibliographystyle{unsrtnat}
\bibliography{ref}

\newpage 
\appendix

\setcounter{page}{1}   

\begin{center}
\huge Supplementary Material
\end{center}

\renewcommand{\thepage}{S\arabic{page}} 

\renewcommand{\thefigure}{S\arabic{figure}}
\setcounter{figure}{0}
\renewcommand{\thesection}{\Alph{section}}
\setcounter{section}{0}
\renewcommand{\thetable}{S\arabic{table}}
\setcounter{table}{0}

\input{appendix/appendix}

\end{document}

%% file: preamble.tex
\usepackage{amssymb, amsfonts, amsthm}
\usepackage{mathtools}
\usepackage{graphicx}
\usepackage{xcolor}
\usepackage{cancel}
\usepackage{listings}
\usepackage[linesnumbered,ruled,vlined]{algorithm2e}
\usepackage{multirow}
\usepackage[shortlabels]{enumitem}
\usepackage{subcaption}
\usepackage[section]{placeins}
\usepackage{natbib}

\usepackage[colorlinks,linktoc=all]{hyperref}
\usepackage{xurl} 
\usepackage[all]{hypcap}
\hypersetup{citecolor=blue}
\hypersetup{urlcolor=blue}
\hypersetup{linkcolor=black}

\usepackage[sort,compress,capitalize,nameinlink]{cleveref}
\creflabelformat{equation}{#1#2#3}
\crefname{equation}{eq.}{eqs.}
\Crefname{equation}{Eq.}{Eqs.}
\Crefname{section}{\S}{\S}

\crefname{appendix}{Supplementary Section}{Supplementary Sections}
\Crefname{appendix}{Supplementary Section}{Supplementary Sections}

\usepackage{makecell}
\usepackage{booktabs}

\SetKwInput{KwInput}{Input}
\SetKwInput{KwOutput}{Output}
\crefname{algocf}{algorithm}{algorithms}
\Crefname{algocf}{Algorithm}{Algorithms}

\DeclareMathOperator*{\argmax}{arg\,max}
\DeclareMathOperator*{\argmin}{arg\,min}

\newcommand*\diff{\mathop{}\!\mathrm{d}}

\newcommand{\pooldist}{g}

\newcommand{\kl}[2]{\mathrm{KL}\left[#1 \,\|\, #2\right]}

\newcommand{\sigmoid}{\textrm{sigmoid}}

\newcommand{\uniform}[1]{\textrm{Uniform}\left(#1\right)}

\newcommand{\supp}[1]{\textrm{supp}\left(#1\right)}

\newcommand{\calD}{\mathcal{D}}
\newcommand{\calE}{\mathcal{E}}

\newcommand{\calA}{\mathcal{A}}

\newcommand{\calI}{\mathcal{I}}
\newcommand{\calZ}{\mathcal{Z}}

\newtheorem{theorem}{Theorem}
\newtheorem{lemma}{Lemma}
\newtheorem{proposition}{Proposition}

\newtheorem{definition}{Definition}

\newtheorem{assumption}{Assumption}

\crefname{assumption}{assumption}{assumptions}
\Crefname{assumption}{Assumption}{Assumptions}
\crefname{definition}{definition}{definitions}
\Crefname{definition}{Definition}{Definitions}
\crefname{proposition}{proposition}{propositions}
\Crefname{proposition}{Proposition}{Propositions}

\newcommand{\BlackBox}{\rule{1.5ex}{1.5ex}}
\renewenvironment{proof}{\par\noindent{\bf Proof }}{\hfill\BlackBox\newline}

\newcommand{\di}[2]{#1^{(#2)}}

\newcommand{\conv}{\stackrel{P}{\to}}

\newtheorem*{repthmconsistency}{Theorem (Posterior consistency)}

\newtheorem*{repthmrate}{Theorem (Posterior contraction rate)}

 \DeclareRobustCommand{\parhead}[1]{\vspace{0.075in} \noindent \textbf{#1}~}

\usepackage{etoolbox}
\usepackage{setspace}


%% file: sections/abstract.tex
\begin{abstract}

  Invariant prediction \citep{peters2016causal} analyzes
  feature/outcome data from multiple environments to identify
  \textit{invariant features}--those with a stable predictive
  relationship to the outcome.  Such features support generalization
  to new environments and help reveal causal mechanisms.  Previous
  methods have primarily tackled this problem through hypothesis
  testing or regularized optimization. Here we develop
  \textit{Bayesian Invariant Prediction (BIP)}, a probabilistic model
  for invariant prediction.  BIP encodes the indices of invariant
  features as a latent variable and recovers them by posterior inference.  Under the assumptions of \citet{peters2016causal},
  the BIP posterior targets the true invariant features.  We prove
  that the posterior is consistent and that greater environment heterogeneity
   leads to faster posterior contraction.  To handle
  many features, we design an efficient variational approximation
  called \textit{BIP-VI}.  In simulations and real data, we find that
  BIP and BIP-VI are more accurate and scalable than existing methods
  for invariant prediction.

\textit{Keywords:} 
Bayesian Modeling; 
Feature Selection; 
Invariant Prediction; 
Variational Inference. 
\end{abstract}

%% file: sections/introduction.tex
\section{Introduction}

An important goal of statistics is to identify variables whose conditional relationship with an outcome remains stable across varying settings and conditions, broadly called \textit{environments}. These variables can support prediction models that generalize to novel situations, and in some settings they can point to causal mechanisms that underlie the data. \looseness=-1

Classical variable selection asks which predictors explain an outcome under the distribution represented by the observed data. In multi-environment data, however, this target can conflate stable relationships with environment-specific associations. A variable may be predictive because it participates in a stable conditional relationship with the outcome, or because it is correlated with the outcome only under a particular site, cohort, time period, experimental condition, or sampling regime. 
For example, in multi-site clinical studies,  
differences across hospitals, including clinical workflows and patient demographics, can induce site-specific effects. As a result, a variable can appear important within observed hospitals under conventional selection while failing to support robust prediction in a new hospital 
\citep{sauer2022leveraging,singh2022generalizability,knight2025fast}. 

In contrast, \textit{invariant variable selection}, first formalized by \citet{peters2016causal}, targets variables whose conditional relationship with the outcome is shared across environments.  Formally, suppose we collect data $\calD :=\{\{ x_{ei}, y_{ei}\}_{i=1}^{n_e}\}_{e=1}^{E}$ from known environments $e=1,\ldots,E$. In this data,  $x$ is a vector of feature variables, $y$ is an outcome of interest, and $i$ indexes data points. The goal is to identify the invariant features of $x$. These are the features such that the conditional distribution of the response is the same across environments, even though the distribution of the features may vary across them. 

Continuing the clinical example, each environment is a hospital, and each data point contains features $x$, such as patient characteristics and clinical measurements, and the treatment outcome $y$. Invariant variable selection seeks features that can reliably predict the outcome across hospitals. This target is useful when a new hospital maintains the same stable prediction mechanism, even if the the remaining variables change. It may also reveal biological mechanisms, rather than artifacts of site-specific clinical practice.

Beyond clinical applications, invariance has been a useful modeling tool in many settings. In manufacturing, environments may be different plants, production lines, or operating conditions; invariant features could correspond to process variables whose relationship with an outcome (like product quality) persists across changes in equipment or operating regime \citep{mo2024sparsity}. In gene perturbation studies, environments are experimental regimes that alter the joint distribution of gene expression and invariant features are candidate regulators of a target gene \citep{meinshausen2016methods}.

In this paper, we introduce \textit{Bayesian Invariant Prediction (BIP)}, which casts invariant variable selection as a Bayesian inference problem. We design a probabilistic model of multi-environment data, where the invariant feature set is encoded as a latent variable, and we infer the invariant set through Bayesian posterior inference. We show that the BIP posterior targets the true invariant set, and provides a model-based quantification of uncertainty about it. Our methodology is in contrast to existing algorithms, which either search over the combinatorial number of candidate invariant sets and use hypothesis testing \citep{peters2016causal}, or frame the problem as a regularized optimization that seeks a point estimate under an invariance regularization \citep{fan2023environment}.

How does BIP work? At its core, \citet{pearl2009causality} makes two assumptions about the data-generating process for invariant feature selection in multi-environment data:
\begin{enumerate}
\item The features $x$  follow different distributions across environments.
\item The outcome $y$ depends on a subset of $x$ in the same way across environments; this subset is the \emph{invariant feature set}.
\end{enumerate}
BIP explicitly bakes these assumptions into a joint distribution of the invariant feature set and multi-environment data. We model the outcome, conditional on the invariant features, to be independent of the environment, and  the per-environment feature distributions. Crucially, we treat the invariant feature set as a latent variable.
Originally, \citet{peters2016causal} motivated invariance through a causal setting where
each environment is an unknown intervention on the features and some features are direct causes of the outcome. Here we adopt the same invariance assumption, without appealing to causality. 

In detail, BIP model encodes the latent invariant feature set through a feature selection vector $z \in \{0,1\}^p$, where each entry $\di{z}{i}$ indicates the inclusion of the $i$-th feature. Let $x^z$ denote the subvector of features where $\di{z}{i} = 1$ and $x^{-z}$ the subvector where $\di{z}{i}=0$. With this notation, the BIP model factorizes as follows
\begin{align}
  \label{eq:model}
  p(z, \calD) =
  p(z)
  \prod_{e=1}^{E}
  \prod_{i=1}^{n_e}
  p_e(x_{ei}^z) \,
  g(y_{ei} \,|\, x_{ei}^z) \,
  p_e(x_{ei}^{-z} \,|\, y_{ei}, x_{ei}^z).
\end{align}
Notice that the conditional  outcome distribution $g$ is  shared across environments, while the feature distributions $p_e$ are allowed to  differ.  

We use this model to calculate the posterior distribution of the invariant set $p(z \,|\, \calD)$. 
We prove that this posterior distribution is consistent, targeting the true invariant features. We also analyze the posterior contraction rate with respect to the per-environment sample size and the number of environments, and characterize its dependence on the prior of $z$ and the heterogeneity of the environments.

Similar to the hypothesis testing method in \citet{peters2016causal}, exact posterior inference for BIP requires searching over $2^p$ possible sets of invariant features. When $p$ is large, exact inference is computationally prohibitive. To address this challenge, we develop \textit{BIP-VI},  a variational inference (VI) algorithm~\citep{jordan1999variational,blei2017variational} to approximate the posterior.  BIP-VI finds a $p$-vector of approximate posterior probabilities that each feature is in the invariant set. Algorithmically, it optimizes a variational objective over the discrete feature selector.  
Compared to exact inference, BIP-VI does not require an exhaustive exponential search.

We study BIP and BIP-VI on both simulated and real data, comparing to existing approaches \citep{peters2016causal,rothenhausler2019causal,fan2023environment}.  The simulation study verifies our theoretical findings of BIP and demonstrates improved accuracy and scalability over existing works. In particular, BIP-VI performs competitively with exact inference in low-dimensional settings, while remaining scalable to high-dimensional cases. In a real-world study, we examine gene perturbation data that includes $p=6,169$ features \citep{kemmeren2014large}. BIP-VI operates on the full feature space and achieves the highest prediction accuracy, while existing methods rely on feature screening, which usually results in either overly conservative predictions or low accuracy.

We summarize our contributions as follows:

\begin{enumerate}[itemsep=6pt, parsep=0pt, topsep=2pt]
\item We design \textit{Bayesian Invariant Prediction (BIP)} for invariant feature selection in 
  multi-environment data. It encodes the invariance assumption and
  enables posterior inference of  the invariant features.

\item We establish theoretical guarantees for BIP, including posterior consistency and contraction rates.

\item We develop \textit{BIP-VI}, a scalable variational inference
  algorithm for high-dimensional settings. It overcomes the
  exponential complexity of exact BIP inference.

\item On simulated and real-world data, we compare BIP with existing approaches and demonstrate improved inference accuracy and scalability.
\end{enumerate}

The paper is organized as follows. In \Cref{sec:method} we state the invariance assumption and introduce BIP, including the
model structure, the exact inference procedure, and a scalable
variational approximation. In \Cref{sec:theory}, we present theoretical results on posterior consistency and contraction rates.
In \Cref{sec:experiments}, we validate BIP on simulations and use BIP to study real data. 
In
\Cref{sec:conclusion}, we summarize BIP and discuss limitations and
future work.

\subsection{Related works}\label{subsec:related-works}

Invariance closely relates to the independent causal mechanism in causal inference literature \citep{peters2017elements,scholkopf2021toward}, which states that the causal mechanism $p(y | x)$ and the cause distribution $p(x)$ are independent under the causal structure $x \to y$. This property implies the possibility of intervening on $p(x)$ while keeping $p(y | x)$ invariant. Invariance has been studied in econometrics under the terms autonomy and modularity \citep{frisch1948autonomy, hoover2008causality}, and later appeared in  computer science as stable and autonomous parent-child relationships in causal graphs \citep[p.~22]{pearl2009causality}. For a historical overview of invariance, see \citet[Chapter 2.2]{peters2017elements} and \citet{buhlmann2020invariance}.

\citet{peters2016causal} is the seminal work that formalized the
concept of invariant prediction for machine learning problems, as well
as connecting it to causal inference. It assumes that the distribution
of the outcome given a subset of features remains the same across
environments, and under certain conditions, these features are the
direct causes of the outcome. It develops a hypothesis testing
approach for identifying such features in linear models, which
requires an exhaustive search over exponentially many
candidates.

Since \citet{peters2016causal}, invariant prediction has become an active area of research. \citet{heinze2018invariant} extends this framework to non-linear settings. \citet{rothenhausler2019causal} proposes an invariance notion based on feature-residual inner products in linear models,
linking it to causal inference under additive
interventions. \citet{fan2023environment} incorporates residual-based
invariance as a regularization term in linear
regression. \citet{wang2024causal} establishes a theoretical
connection between invariant prediction models and causal
outcome models, leading to a computationally efficient algorithm for
causal discovery. \citet{gu2025fundamental} shows that solving
for exact invariance is NP-hard and introduces an objective that
interpolates between exact invariance and predictive performance. In concurrent research, \citet{madaleno2025bayesian}  leverages a Bayesian hierarchical model to explicitly test the invariance of causal mechanisms.

This paper offers a probabilistic perspective on invariant prediction, with goals closely aligned to \citet{peters2016causal} and \citet{fan2023environment}. As we will see below, this perspective opens the door to both new algorithms and new theoretical understanding.

%% file: sections/method.tex
\section{Bayesian Invariant Prediction}
\label{sec:method}

\subsection{Multi-environment data and the invariance assumption}
\label{subsec:multi-env-data-setup}

Consider a dataset of feature-outcome pairs that are organized into
\textit{multiple environments}, $\calD :=\{\{ x_{ei},
y_{ei}\}_{i=1}^{n_e}\}_{e=1}^{E}$. In each pair, $x$ is a $p$-dimensional vector and $y$ is a general outcome. Each environment $e$ indexes a distinct data distribution $p_e(x,y)$, and $x_{ei}, y_{ei}$ were drawn IID from this distribution. In practice, environments are groups of observations collected under different conditions, such as distinct sites, experimental regimes, or time periods, and are fixed by the study design. 

Our main assumption is that some of the features are \textit{invariant}---the conditional distribution of the outcome given these features remains unchanged across all environments \citep{peters2016causal}. Formally, we use a feature selector $z \in \{0,1\}^p$ to partition the features into the selected features $x^z = (x^{(i)})_{\,i:\, z^{(i)}=1}$ and the unselected features $x^{-z} = (x^{(i)})_{\,i:\, z^{(i)}=0}$. With this notation, we state the invariance assumption formally:
\begin{assumption}[Invariance]
  \label{ass:invariance}
  For a collection of environments $\calE$, there exists a feature
  selector $z^* \in \{0,1\}^p$ such that
  $p_e(y \mid x^{z^*}) = p_{e'}(y \mid x^{z^*})$ for all $e, e' \in \calE$.
\end{assumption}

We call $z^*$ the \textit{(true) invariant feature selector} and $x^{z^*}$ the \textit{(true) invariant features}. More generally, we say a selector $z$ is \emph{invariant} if $p_e(y \mid x^z)$ is identical across all $e \in \calE$; \Cref{ass:invariance} posits that at least one invariant selector exists. By \Cref{ass:invariance} and the chain rule, each environment's joint distribution factorizes,
\begin{align}
  \label{eq:per-environment-joint}
  p_e(x,y)
  &=
    p_e(x^{z^*} ) p_*(y \mid x^{z^*}) p_e(x^{-z^*} \mid x^{z^*}, y).
\end{align}
The invariant conditional $p_*(y \mid x^{z^*}) \equiv p_e(y \mid
x^{z^*})$ for all environments $e \in \calE$.

The problem is that we do not observe the invariant feature selector $z^*$. Our goal is to infer $z^*$ from the observed multi-environment data.

\subsection{A Bayesian model of invariant prediction}

We propose Bayesian Invariant Prediction (BIP) to infer the true invariant features $z^*$. Our model places a prior distribution $p(z)$ over feature selectors $z \in \{0,1\}^p$ and defines a likelihood based on  the invariance assumption.

For now, we assume the local joint distributions $\{p_e(x,y)\}_{e=1}^{E}$ are known. Consequently, for any feature selector $z$, we can calculate the \textit{local marginal} $p_e(x^z)$, \textit{local conditionals} $p_e(y|x^z)$, and \textit{local feature conditionals} $p_e(x^{-z} \mid x^z, y)$, all derived from the local joint $p_e(x,y)$. In practice, we estimate them with a sufficiently large dataset. 

To build our likelihood, we also define a \textit{pooled conditional distribution} of the outcome given candidate features $x^z$,  derived from the collection of local distributions.

\begin{definition}[Pooled conditional]
\label{def:pooled-conditional}
Given a feature selector $z$ and a collection of local distributions $\{p_e(x,y)\}_{e \in \calE}$, define the \textit{pooled conditional} as:
\[
\pooldist(y \mid x^z) := \frac{\sum_{e \in \calE}\int p_e(x,y)\diff x^{-z}}{\sum_{e \in \calE}p_e(x^z)}.
\]
\end{definition}
The pooled conditional is a valid distribution by definition. It characterizes the distribution of outcome $y$ given candidate features $x^z$, pooled across environments. In BIP, the pooled conditional will serve as the invariant distribution of the outcome. Its key property is that it matches the local conditional distributions $p_e(y \mid x^z)$ simultaneously across all environments if and only if $z$ is invariant. Therefore, it will help the posterior center on the true invariant feature selector.  We state this formally below.

\begin{proposition}[Identifiability of invariant features]
\label{prop:faithfulness}
Under \Cref{ass:invariance}, the equality
$\pooldist(y \mid x^z) = p_e(y \mid x^z) \, \forall e \in \calE$ 
holds if and only if $z$ is an invariant selector. If $z^*$ is unique, then this equality holds if and only if $z = z^*$.
\end{proposition}

\begin{proof}
If $z$ is invariant, then by definition $p_e(y \mid x^z)$ is identical across environments. The pooled conditional $\pooldist(y \mid x^z)$ is a $p_e(x^z)$-weighted average of these identical local conditionals, so it equals their common value, giving the equality. Conversely, if $\pooldist(y \mid x^z)$ matches each  $p_e(y \mid x^z)$, then these local conditionals must themselves be identical across environments, so $z$ is invariant. When $z^*$ is unique, the only such selector is $z^*$.
\end{proof}

With the pooled conditional in hand, we now define the BIP model:
\begin{itemize}
\item Draw invariant feature selector $z \sim p(z)$.
\item For  $e=1,\cdots, E \in \calE$, and each observation $i=1,\dots,n_e$:
\begin{enumerate}
\item Draw invariant features $x_{ei}^z \sim p_e(x_{ei}^z)$.
\item Draw outcome $y_{ei} \mid x_{ei}^z \sim \pooldist(y_{ei} \mid x_{ei}^z)$.
\item Draw remaining features $x_{ei}^{-z} \mid x_{ei}^z,y_{ei} \sim p_e(x_{ei}^{-z}\mid x_{ei}^z,y_{ei})$.
\end{enumerate}
\end{itemize} 
The BIP model encodes invariance. Step 2 explicitly uses the pooled conditional, thereby enforcing a single conditional distribution for $y \mid x^z$ across environments.  The joint distribution of latent and observed variables was introduced earlier in \Cref{eq:model}.

Under \Cref{ass:invariance}, and assuming the invariant selector $z^*$ is unique, the BIP's posterior will concentrate around $z^*$. Intuitively, for any $z \neq z^*$,  the pooled conditional $\pooldist(y \mid x^z)$ cannot match all the true local distributions $p_e(y \mid x^z)$, resulting in a lower data likelihood and hence lower posterior probability. But when $z = z^*$, the pooled conditional matches all local distributions, resulting in a higher data likelihood and hence posterior probability. To see this observation more concretely, we write the posterior distribution explicitly as a product of prior and  likelihood ratios -- each  between the pooled conditional and a local conditional. This expression is a key property of our model.

\begin{proposition}[Posterior expression]
\label{prop:posterior-expression}
Under the BIP model in  \Cref{eq:model}, the posterior distribution simplifies to:
\begin{align}
  \label{eq:simplified_posterior}
  p(z \mid \calD)
  \propto p(z)\prod_{e=1}^{E}\prod_{i=1}^{n_e}\frac{\pooldist(y_{ei}\mid x_{ei}^z)}{p_e(y_{ei}\mid x_{ei}^z)}.
\end{align}
\end{proposition}

\begin{proof}
  Start from the joint distribution under the model:
  \begin{align*}
    p(z, \calD)
    & =
      p(z)\prod_{e=1}^E \prod_{i=1}^{n_e}
      p_e(x_{ei}^z) \pooldist(y_{ei}\mid x_{ei}^z)
      p_e(x_{ei}^{-z}\mid x_{ei}^z,y_{ei}).
  \end{align*}
  Multiplying a constant of 1 in the form of $p_e(y_{ei}\mid x_{ei}^z)/p_e(y_{ei}\mid x_{ei}^z)$, we obtain
  \begin{align*}
       p(z, \calD)
    & =
      p(z)\prod_{e=1}^E \prod_{i=1}^{n_e}
      p_e(x_{ei}^z) \pooldist(y_{ei}\mid x_{ei}^z)
      p_e(x_{ei}^{-z}\mid x_{ei}^z,y_{ei}) \frac{p_e(y_{ei}\mid x_{ei}^{z})}{p_e(y_{ei}\mid x_{ei}^{z})}. 
  \end{align*}
  Note that
  $  p_e(x_{ei}^z)
    p_e(y_{ei}\mid x_{ei}^z)
    p_e(x_{ei}^{-z}\mid x_{ei}^z,y_{ei})
    =
    p_e(x_{ei}, y_{ei}).$ 
  So,
  \begin{align}
    \label{eq:joint-expression}
    p(z, \calD)
    & =
      p(z) \prod_{e=1}^{E}
      \prod_{i=1}^{n_e} p_e(x_{ei},y_{ei})
      \frac{\pooldist(y_{ei}\mid x_{ei}^z)}{p_e(y_{ei}\mid x_{ei}^z)}.
  \end{align}
  The local joint distribution $p_e(x_{ei},y_{ei})$ is constant with
  respect to $z$. Therefore we obtain the simplified posterior of
  \Cref{eq:simplified_posterior}.
\end{proof}

This expression highlights the role of the ratio between the pooled conditional and the local conditionals, and shows that the posterior centers on the true invariant feature selector. At $z=z^*$, Proposition~\ref{prop:faithfulness} guarantees equality between the pooled and local conditionals, maximizing this ratio. When $z^*$ is the unique invariant selector, any $z \neq z^*$ makes the pooled conditional $\pooldist(y \mid x^z)$ differ from the true local distributions $p_e(y\mid x^z)$, decreasing the ratio, so the posterior concentrates around the true $z^*$. We formalize this intuition through posterior consistency and contraction results in \Cref{sec:theory}.

Note that this likelihood ratio precisely targets the invariance criterion in the hypothesis-testing framework of \citet{peters2016causal}, where a ratio of one corresponds to the null hypothesis of invariance; here the same ratio instead arises naturally from the invariance assumption and a probabilistic model. Although BIP and ICP share this criterion, their inference procedures differ: ICP accepts or rejects each candidate set and reports the \emph{intersection} of those that pass, whereas BIP bakes the ratios into a posterior over candidate sets. As a result, ICP is conservative---its intersection often collapses to the empty set---whereas BIP-VI recovers more invariant features from posterior summaries, a difference borne out by the empirical studies of \Cref{subsec:synthetic-comparison,sec:gene-study}.

\subsection{Exact posterior inference for BIP}
\label{subsec:exact-posterior-inference}

We now turn to calculating the posterior. In practice we do not have access to the true distributions $p_e(x,y)$ and must estimate them. Rather than estimate the full joint distributions, we note that the simplified posterior in \Cref{eq:simplified_posterior} depends only on the local conditionals $p_e(y \mid x^z)$ and the pooled conditional $\pooldist(y \mid x^z)$, so we model these conditionals directly. Specifically, let $\mathcal{P}_{y \mid x^z}$ be a conditional model class indexed by $z$; we estimate the local and pooled conditionals, respectively, via maximum likelihood:
 \begin{align}
   \hat p_e(y \mid x^z)
   &:=
     \argmax_{\tilde p (y \mid x^z) \in \mathcal{P}_{y\mid x^z}} \sum_{i=1}^{n_e} \log \tilde p (y_{ei} \mid x_{ei}^z), \quad \forall e, z \label{eq:estimated-local-conditional} \\
   \hat \pooldist(y \mid x^z) &:= \argmax_{\tilde p (y \mid x^z) \in \mathcal{P}_{y \mid x^z} } \sum_{e =1}^E \sum_{i=1}^{n_e} \log \tilde p (y_{ei} \mid x_{ei}^z), \quad \forall z. \label{eq:estimated-pooled-conditional}
 \end{align}

For example, suppose $\mathcal{P}_{y \mid x^z}$ is the linear Gaussian family $\{\mathcal{N}(y \mid \beta^\top x^z + b, \sigma^2)\}$. Then \Cref{eq:estimated-local-conditional} yields $\hat p_e(y \mid x^z) = \mathcal{N}(y \mid \hat\beta_e^\top x^z + \hat b_e, \hat\sigma_e^2)$, where $(\hat\beta_e, \hat b_e, \hat\sigma_e^2)$ are the maximum likelihood estimates on environment $e$'s data $\{(x_{ei}, y_{ei})\}_{i=1}^{n_e}$; similarly, \Cref{eq:estimated-pooled-conditional} fits the same Gaussian family on data pooled across all environments.

We then compute the posterior distribution explicitly by iterating over all candidate feature selectors $z$ according to \Cref{eq:simplified_posterior}:
\begin{align}
  \hat p(z\mid \calD )
  & \propto p(z)
    \prod_{e \in \calE}
    \prod_{i=1}^{n_e}
    \frac{\hat \pooldist(y_{ei} \mid x_{ei}^z)}
    {\hat p_e(y_{ei} \mid x_{ei}^z)}.
    \label{eq:posterior-estimated-proportionality}
\end{align}

The exact posterior inference procedure is summarized in \Cref{alg:exact}. The algorithm  explicitly enumerates all candidate feature selectors, estimates the local and pooled conditionals, and calculates the posterior probabilities accordingly.

\begin{algorithm}[!t]
\caption{BIP: exact posterior inference for invariant features}
\label{alg:exact}
\KwInput{Dataset $\calD= \{ (x_{ei}, y_{ei})\}$, prior $p(z)$}
\KwOutput{Posterior distribution $\hat p(z \mid \calD)$ over feature selectors}
\For{each selector $z$ in the support of $p(z)$ }{
    Estimate local conditionals $\{ \hat p_e(y \mid x^z)\}_{e=1}^E$ and pooled conditional $\hat \pooldist(y \mid x^z)$ via \Cref{eq:estimated-local-conditional} and \Cref{eq:estimated-pooled-conditional}\;

    Compute the likelihood ratio:
    $\hat \Lambda(z) \leftarrow \prod_{e=1}^E \prod_{i=1}^{n_e}  \frac{\hat \pooldist(y_{ei} \mid x_{ei}^z)}{\hat p_e(y_{ei} \mid x_{ei}^z)}$\;
}
Normalize to obtain the posterior:
$\hat p(z \mid \calD) = \dfrac{p(z)\, \hat \Lambda (z)}{\sum_{z'} p(z')\, \hat \Lambda (z')}.$
\end{algorithm}

This procedure can be expensive. When the prior $p(z)$ has full support on $\{0,1\}^p$, exact inference has complexity $O(2^p \cdot c(\calD, p))$, where $c(\calD, p)$ is the cost of estimating and evaluating the conditional models in \Cref{eq:estimated-local-conditional,eq:estimated-pooled-conditional} on data with up to $p$ features. Restricting the support of $p(z)$ to subsets of size at most $p_{\max}$ reduces this to $O(p^{p_{\max}}\cdot c(\calD, p_{\max}))$, which may still be prohibitive for large $p$.

\subsection{Variational inference for BIP}
\label{subsec:vi}

\begin{algorithm}[t]
\caption{BIP-VI: variational inference for invariant features}
\label{alg:viml}
\KwInput{Dataset $\calD$, iterations $T$, number of gradient samples $M$, learning rate scheduler $r(\cdot)$}
\KwOutput{Variational posterior $q_{\phi_T}(z)$}
Initialize variational parameter $\phi_0$ \\
\For{$t = 1$ \KwTo $T$}{
    Compute $M$ unbiased gradient estimates of the ELBO using U2G (\Cref{alg:u2g}, see the supplementary materials):
    $\hat h_m \approx \nabla_\phi \mathcal{L}(\calD,\phi_{t-1})$ for  $m=1,\dots,M$ \\
    Update variational parameters:
    $
      \phi_{t}
      \leftarrow
      \phi_{t-1} + r(t) \cdot \frac{1}{M}\sum_{m=1}^M \hat h_m
    $
}
\end{algorithm}

Exact posterior inference enumerates all possible invariant selectors, which is infeasible for large feature sets. We therefore turn to variational inference \citep[VI,][]{blei2017variational}, which approximates the posterior with a simpler distribution, to develop BIP-VI.

We posit a mean-field family of approximate posterior distributions of $z$:
 \begin{align*}
 q_\phi(z) & \coloneqq \prod_{j=1}^p \textrm{Bernoulli}\left(z^{(j)} \mid \textrm{sigmoid}(\phi^{(j)})\right),
 \end{align*}
where $\textrm{sigmoid}(\cdot) \coloneqq \frac{\exp(\cdot)}{1 + \exp(\cdot)}$, and $\phi \in \mathbb{R}^p$ are the variational parameters. We then fit these parameters by maximizing the evidence lower bound (ELBO),
\begin{align}
  \label{eq:elbo}
  \begin{split}
 \mathcal{L}(\calD, \phi) &= \mathbb{E}_{q_\phi(z)}\left[\log p(z, \calD) - \log q_\phi(z)\right] \\
 &= \mathbb{E}_{q_\phi(z)}\left[\log p(z) + \sum_{e=1}^E \sum_{i=1}^{n_e}\log \frac{\pooldist(y_{ei}\mid x_{ei}^z)}{p_e(y_{ei}\mid x_{ei}^z)} - \log q_\phi(z)\right] + C,
 \end{split}
 \end{align}
where $C = \sum_{e=1}^E\sum_{i=1}^{n_e}\log p_e(x_{ei},y_{ei})$ is a constant with respect to $\phi$. The second line substitutes the joint distribution from \Cref{eq:joint-expression} and absorbs the $\phi$-independent log local-joint terms $\sum_{e,i}\log p_e(x_{ei},y_{ei})$ into the constant $C$. The optimal variational parameters minimize the KL divergence to the exact posterior.

We optimize the ELBO by stochastic gradient ascent, where the local and pooled conditionals, $\hat p_e(y \mid x^z)$ and $\hat \pooldist(y \mid x^z)$, are estimated from data as in \Cref{eq:estimated-local-conditional,eq:estimated-pooled-conditional}. We use Monte Carlo estimates of the ELBO gradient, a form of black-box variational inference~\citep{kingma2014vae,ranganath2014black}. \Cref{alg:viml} summarizes this procedure.  The complexity of BIP-VI is $O(T \cdot M \cdot c(\calD, p_{\max}))$ for $T$ iterations and $M$ gradient samples, where each gradient sample evaluates the conditionals on a candidate subset of size at most $p_{\max}$, the subset-size cap imposed by the prior $p(z)$. 

A technical challenge in BIP-VI is computing gradients of the ELBO, which involve expectations over the discrete variational distribution $q_\phi(z)$. We use the unbiased U2G estimator \citep{yin2020probabilistic}; see \Cref{alg:u2g} in the supplementary materials. Other discrete gradient estimators, such as those surveyed by \citet{mohamed2020monte}, are equally compatible. Additional implementation details are provided in \Cref{app:vi:optimization}.

%% file: sections/theory.tex
\section{Theory}
\label{sec:theory}

We analyze the asymptotic behavior of the exact posterior $\hat p(z \mid \calD)$ (\Cref{alg:exact}, under estimated models $\hat{p}_e$). For simplicity we take a common per-environment sample size $n_e=n$, and to let the number of environments $E$ grow, we draw each environment from a uniform distribution $p(e)$ over the environments of interest $\calE$. Throughout, $\conv$ denotes convergence in probability, $\supp{p(z)}$ the support of $p(z)$, and $\forall_p z$ ``for every $z$ with positive prior mass.'' We outline our theoretical findings as follows: 
\begin{itemize}[itemsep=6pt, parsep=0pt, topsep=2pt]
  
\item \Cref{subsec:theory-main}: under suitable assumptions the posterior concentrates on $z^*$ as $E$ and $n$ grow, and under additional regularity conditions it contracts exponentially in $n$ and $E$, faster under stronger prior knowledge and greater environment heterogeneity.
  
\item \Cref{subsec:theory-discussion}: we extend these results to finite \( n \) or \( E \), and show that relaxing the assumptions makes the posterior concentrate on the most ``approximately invariant'' selectors under the given model and prior.
  
\end{itemize}

\subsection{Main results}\label{subsec:theory-main}

We make the following assumptions: 
\begin{assumption}[Uniqueness]\label{ass:uniquness}
  In \Cref{ass:invariance}, the existence of $z^*$ is unique.
\end{assumption}

\begin{assumption}[Prior positivity]\label{ass:positive-prior} In
  \Cref{ass:invariance},  $z^*$ has positive prior mass.
\end{assumption}

\begin{assumption}[Estimation consistency] \label{ass:likelihood-consistency}
$\forall_p z$, as $n \to \infty$ and $E \to \infty$,
\begin{align*}
\frac{1}{nE}\sum_{i=1}^n \sum_{e=1}^E \log \hat p_e(y_{ei} \mid x_{ei}^z) &\conv 
\mathbb{E}_{p(e) p_e(y,x^z)}\left[ \log p_e(y \mid x^z) \right], \\
\textrm{and }\frac{1}{nE}\sum_{i=1}^n \sum_{e=1}^E \log \hat g(y_{ei} \mid x_{ei}^z) & \conv 
\mathbb{E}_{p(e)p_e(y , x^z)}\left[ \log g(y \mid x^z) \right].  
\end{align*}
\end{assumption}

With these assumptions, we now state our first theorem. Given infinite data and environments, BIP correctly identifies the invariant $z^*$.

\begin{theorem}[Posterior consistency]\label{thm:posterior-consistency}
Suppose \Cref{ass:invariance,ass:uniquness,ass:positive-prior,ass:likelihood-consistency} hold, and take the limits $n\to \infty$ and  $E \to \infty$. The BIP posterior satisfies the following properties: 
\begin{itemize}
    \item consistency of the posterior mode: $\hat z_{n,E}: = \argmax_z \hat p(z \mid \calD) \stackrel{P}{\to} z^*$
    \item consistency of the posterior at $z^*$: $\hat p(z^* \mid \calD) \stackrel{P}{\to} 1. $
\end{itemize}
\end{theorem}

The proof is provided in \Cref{app:subsec:proof-consistency}. 
\Cref{thm:posterior-consistency} relies on key assumptions:
\Cref{ass:invariance} asserts that the true data generating process
satisfies invariance; \Cref{ass:uniquness} guarantees
identifiability of $z^*$; \Cref{ass:positive-prior} requires the prior
is well-specified; \Cref{ass:likelihood-consistency} ensures
consistent estimation of the likelihood. 

Next, we characterize the posterior contraction rate. First we make an assumption.
\begin{assumption}[Finite variance of log-likelihood ratio]\label{ass:finite-variance}
 $\forall_p z$,  the log-likelihood ratio $\log p_e(y \mid x^z) - \log \pooldist (y\mid x^z)$ has finite variance, that is, 
 \begin{align*} 
v(z):=\textrm{Var}_{p(e) p_e(y,x^z)} \left( \left[ \log p_e(y \mid x^z) - \log \pooldist (y \mid x^z) \right] \right)  < \infty.
\end{align*}
\end{assumption}
This regularity condition enables control over posterior contraction through Chebyshev's inequality. To state the contraction rate, we define the total variational distance $\textrm{TV}(\cdot, \cdot)$ between the posterior and the Dirac measure centered at the true $z^*$ as follows,
\begin{align*}
     \textrm{TV}\Big( \hat p(z \mid \calD), \delta_{z^*}(z) \Big) &= \frac{1}{2} \left( \sum_{z \neq z^*} | \hat p(z \mid \calD)-0| + | p(z^* | \calD) - 1| \right) = 1- \hat p(z^* | \calD). 
\end{align*}

\begin{theorem}[Posterior contraction rate]\label{thm:rate}
Given \Cref{ass:invariance,ass:uniquness,ass:positive-prior,ass:likelihood-consistency,ass:finite-variance}, 
there exists a sequence $\epsilon_{n,E} = O(R \cdot e^{-\kappa nE \mu_{\min} })$ such that 
\begin{align*}
    P\Bigg(\textrm{TV}\Big( \hat p(z \mid \calD), \delta_{z^*}(z) \Big) >  \epsilon_{n,E}  \Bigg)\to 0 
\end{align*}
as $n\to \infty$ and $E \to \infty$, where   $\kappa$ is a fixed value in $(0,1)$, and 
\begin{align}
    R &:= \left(\max_{z\neq z^*} \frac{p(z)}{p(z^*)} \right)\cdot | \supp{p(z)}|, \label{eq:def-R} \\
    \mu(z) &:=\mathbb{E}_{p(e)p_e(x^z)} \left[\kl{ p_e(y \mid x^z)}{ \pooldist (y\mid x^z))}\right], \label{eq:def-mu} \\ 
    \mu_{\min} &:= \min_{z \neq z^*, p(z)>0} \mu(z) \label{eq:def-mu-min}. 
\end{align}
\end{theorem}

 The proof is provided in \Cref{app:subsec:proof-rate}. \Cref{thm:rate} shows that the BIP posterior contracts around 
 $z^*$ exponentially fast, with respect to $n$ and $E$.  It also 
reveals  two key factors that influences the contraction rate. 

The first is the prior, which influences the term  $R$ in \Cref{eq:def-R}.  An informative prior that assigns lower
probability to non-invariant $z$s will lead to smaller value of 
$\max_{z\neq z^*} \frac{p(z)}{p(z^*)}$ (or $| \supp{p(z)}|$), and hence a smaller $R$ and faster contraction. In contrast, a prior fully supported on
$\{0,1\}^p$ results in an exponential dependence of $R$ on the
dimensionality $p$, since $|\supp{p(z)}|=2^p$. This analysis
highlights the challenge of high-dimensional invariant prediction problems and
the importance of an informative prior in that setting.
    
The second factor is the heterogeneity of environments, captured by  $\mu_{\min}$ in \Cref{eq:def-mu-min}. $\mu_{\min}$ is defined as
minimum of $\mu(z)$ (\Cref{eq:def-mu}) among non-invariant $z$s, where each $\mu(z)$ measures the discrepancy between
local and pooled conditionals given $z$. By \Cref{ass:invariance}, 
$\mu(z^*)=0$ and $\mu(z)>0$ for $z\neq z^*$. When environments are
more similar,  local conditionals $p_e(y \mid x^z)$ tend to align and hence closer to the pooled
conditional $\pooldist (y \mid x^z)$, yielding smaller $\mu(z)$. As a result, $\mu_{\min}$ is smaller, and the posterior contracts more slowly.  Conversely, a
larger $\mu_{\min}$ suggests greater environment heterogeneity, leading to faster posterior contraction. We empirically confirm this relationship in \Cref{subsec:synthetic-theory}, where we manually increase the environment 
heterogeneity level and observe faster contraction.

\subsection{Finite-sample extensions and relaxations of assumptions}
\label{subsec:theory-discussion}

Our main results consider the limit of infinite samples and environments. We now discuss finite-sample extensions and relaxations of core assumptions, and suggest practical guidelines. Formal statements are in \Cref{app:subsec:theory-extensions,app:subsec:justificiation-discussion}.

\parhead{Finite-sample extensions.}
With a fixed set of $E$ environments and per-environment sample size $n \to \infty$, the consistency and rate results of \Cref{subsec:theory-main} continue to hold (\Cref{thm:posterior-consistency-fiex-E,thm:rate-fixed-E}), provided $z^*$ stays identifiable from the observed environments.

Conversely, with finite per-environment samples but a growing number of environments ($E \to \infty$), each local conditional $\hat p_e(y \mid x^z)$ is estimated from limited data and is therefore biased. Consistency continues to hold as long as this bias is smaller than the heterogeneity signal $\mu_{\min}$ defined in \Cref{eq:def-mu-min}   (\Cref{thm:posterior-consistency-fixed-n,thm:rate-estimated-wrt-E}). The bias shrinks as the per-environment sample size $n$ grows, so it falls below $\mu_{\min}$ once $n$ is moderately large.

\parhead{Relaxations of core assumptions.}
In practice the assumptions behind \Cref{thm:posterior-consistency} rarely hold exactly. \Cref{thm:posterior-consistency-general} covers relaxations of each core assumption:
\begin{itemize}[leftmargin=*]
\item \textit{Invariance (\Cref{ass:invariance}).} Recall
   $\mu(z)$  in \Cref{eq:def-mu} that measures the discrepancy between local and pooled conditionals given $z$, or equivalently the violation level of invariance. If
  invariance holds exactly, $\mu(z)$ reaches its minimum at
  $z^*$, with zero discrepancy $\mu(z^*)=0$. When invariance is violated, the posterior concentrates on the minimizer, $\argmin_{z \in \supp{p}} \mu(z)$,
  which is the most
  \emph{approximately invariant} solution.

\item \textit{Uniqueness (\Cref{ass:uniquness}).}
 When several feature selectors are invariant, the posterior spreads its mass over them. This is more likely when the environments are few: fewer environments tend to admit more invariant sets, and adding environments reduces the number of invariant sets. \Cref{app:subsec:synthetic-uq} gives an empirical illustration on synthetic data, where the posterior is multi-modal under two environments and concentrates on the single remaining one once a third environment is added.

\item \textit{Prior positivity
    (\Cref{ass:positive-prior}).}  If the prior excludes the true $z^*$,  then the
  posterior concentrates on the most approximately invariant solution within the prior support. 
\item \textit{Estimation consistency
    (\Cref{ass:likelihood-consistency}).} When the conditional model is
  misspecified, the estimated local and pooled conditionals no longer
  converge to the true distributions but to their best-fitting
  approximations in the assumed model class. The posterior then concentrates
  on the minimizer of the discrepancy between these approximations.
\end{itemize}

In all four cases the posterior concentrates on the most approximately invariant selectors representable under the chosen prior and model class, reflecting the full landscape of invariant solutions consistent with these modeling choices. \Cref{thm:posterior-consistency} is the special case in which all assumptions hold and this landscape collapses to a unique $z^*$.

\parhead{Practical guidelines.}
These analyses suggest a few guidelines. When per-environment data are limited, the conditional models must stay estimable: we can cap the number of active features at a reasonable $p_{\max}$ in the prior, and if the study design yields too few samples per experimental condition, we can merge similar conditions into one environment to increase estimation power at the cost of fewer environments (the gene study in \S\ref{sec:gene-study} is an example, where all gene perturbation samples are pooled into one environment). Having fewer environments has two consequences. First, the heterogeneity signal $\mu_{\min}$ (\Cref{eq:def-mu-min}) tends to be smaller, so the posterior contracts more slowly (\Cref{thm:rate}). Second, multiple invariant solutions become plausible, yielding a multimodal posterior (\Cref{thm:posterior-consistency-general}). In either case the posterior carries genuine uncertainty, and reporting summaries such as marginal inclusion probabilities is more informative than a point estimate.

%% file: sections/experiments.tex
\section{Experiments}\label{sec:experiments}

\input{sections/exp-method-summary}

\input{sections/synthetic-exp}
\input{sections/gene-exp}


%% file: sections/exp-method-summary.tex
\subsection{Comparison methods}\label{subsec:comparison-methods}

We consider the following invariant prediction methods, including our BIP and BIP-VI. In all settings, the conditional outcome is modeled as a linear Gaussian model.

\begin{enumerate}[leftmargin=*]

\item \textit{Invariant causal prediction (ICP)~\citep{peters2016causal}}: ICP enumerates candidate invariant feature sets, and for each set, tests the null hypothesis that the conditional outcome distribution is the same across all environments. It returns the intersection of all sets that are not rejected at significance level $\alpha$. 

\item \textit{Hidden-ICP~\citep{rothenhausler2019causal}}:  Different from the distributional invariance (\Cref{ass:invariance}), Hidden-ICP instead considers 
\emph{inner-product invariance}, assuming $\mathbb{E}[x (y - x^\top \beta)]$  is constant across environments. $\beta$ is estimated by a least-square procedure with the Gram matrices replaced by differences of Gram matrices in different environments.  The invariant features are selected at significance level $\alpha$. 

\item \textit{EILLS~\citep{fan2023environment}}: A regularized linear regression method that introduces a penalty term encoding a similar inner-produce invariance condition as in Hidden-ICP. The invariance regularization strength~$\gamma$ controls the trade-off between predictive fit and cross-environment stability: larger~$\gamma$ introduces stronger invariance penalty.

\item \textit{BIP [this paper]}: Our method with exact posterior inference (\Cref{alg:exact}). BIP enumerates all candidate invariant selectors and computes the exact posterior. Features are selected based on posterior summaries, e.g. posterior mode or by thresholding the marginal posterior inclusion probability at a level~$t$. 

\item \textit{BIP-VI [this paper]}: Our method with variational inference (\Cref{alg:viml}). Features are selected based on variational posterior summaries, as in the exact inference. 
\end{enumerate}

\parhead{Scaling in high dimensions.} BIP (with exact posterior inference), ICP, and EILLS require enumerating all candidate invariant sets, which is infeasible for large~$p$. Hidden-ICP requires Gram matrix inversion, which is intractable when $p$ exceeds the number of observations.  To apply these methods in high dimensions, we first run Lasso with cross-validation and select the top 10 features by absolute coefficient magnitude. If Lasso returns no nonzero coefficients, we instead select the top 10 features by marginal correlation with the outcome. We label these screened versions as \textit{BIP-screen, ICP-screen, EILLS-screen}, and \textit{Hidden-ICP-screen}. 

In contrast, BIP-VI operates on full features by optimizing a variational objective, without requiring pre-screening.
We use a uniform prior over selectors with at most $p_{\max}$ active features, where $p_{\max}$ is set per experiment setting. We choose $p_{\max} < n_e$ for every environment $e$ so that each local conditional model $p_e(y \mid x^z)$ remains estimable. 

\parhead{Computational comparison.} We report a full wall-clock comparison in \Cref{app:sec:compute}. Consistent with its theoretical complexity (\Cref{subsec:vi}), BIP-VI's cost is governed by its optimization budget rather than the number of candidate invariant feature sets. BIP-VI takes longer than baselines that operate on a small screened feature set, but those methods could not enumerate the much larger search space BIP-VI considers.

%% file: sections/synthetic-exp.tex
\subsection{Synthetic study}\label{sec:synthetic}

We evaluate the BIP framework with simulated data, where the true invariant features are known by construction. This setting allows us to  verify our theoretical results, measure each method's accuracy against the ground truth, and compare our methods to existing approaches. All experimental details are provided in \Cref{app:sec:synthetic-data}. We summarize our findings as follows:

\begin{itemize}[leftmargin=*]

\item In \Cref{subsec:synthetic-theory} we empirically verify the theory in \Cref{subsec:theory-main} with $p=3$ features where exact inference for BIP is tractable. Our results confirm posterior consistency and contraction rates as the number of environments $E$ and per-environment sample size $n$ increase.

\item In \Cref{subsec:synthetic-comparison} we compare BIP and BIP-VI to  existing methods for invariant prediction. We study both a low-dimensional setting ($p=10)$ and a  high-dimensional setting ($p=450$). 
We show that BIP and BIP-VI recover the true invariant feature selector in low dimensions;  and for high-dimensions, BIP-VI is more scalable and accurate than existing methods and BIP with exact inference which require feature screening. 
\end{itemize}
\input{figures/fig_synthetic_dgp}

\parhead{Data generation.}
We simulate multi-environment data from a linear structural causal model with Gaussian noise, a standard setup in the invariance literature \citep{peters2016causal, fan2023environment, rothenhausler2019causal}. This setup reflects settings where variables are related through a sequential mechanism and external perturbations modify some, but the outcome generating mechanism remains unchanged. In gene regulation, for example, gene knockouts alter the expression distributions of the perturbed genes while leaving the outcome gene's regulatory mechanism intact.

We draw a random ordering of the features and outcome, permuting $(x^{(1)}, \ldots, x^{(p)}, y)$ into $(t^{(1)}, \ldots, t^{(p+1)})$, and generate each variable in turn conditional on its predecessors:
\begin{align}\label{eq:synthetic-data-joint-factorization}
    p_e(x,y) = \prod_{i=1}^{p+1} p_e(\di{t}{i} \mid \di{t}{1:i-1}), \qquad \forall e=1,\cdots, E,
\end{align}
where each conditional is a linear model with additive Gaussian noise. Within each run, the ordering is sampled once and held fixed across all environments. We first specify all conditionals in an observational environment ($e=1$). For each interventional environment ($e \geq 2$), we randomly select a subset of features to \textit{intervene} on by modifying the parameters in their conditional distributions, and keep the rest of the conditionals the same as in $e=1$. The outcome $y$ is never intervened on, so its conditional stays fixed across environments, and we treat all variables preceding $y$ in the ordering as the invariant features. \Cref{fig:synthetic-dgp} illustrates an example in which $(t^{(1)}, \cdots, t^{(6)})=(x^{(2)}, x^{(4)}, y, x^{(1)}, x^{(3)}, x^{(5)})$, so that $x^{(2)}, x^{(4)}$ are the invariant features.
We modulate environment heterogeneity by varying the \textit{intervention strength}, the expected fraction of features intervened on.

BIP and BIP-VI adopt a uniform prior over $\{0,1\}^p$ in low dimensions ($p=3, 10$). In high dimensions ($p=450$), BIP uses a uniform prior over the screened feature space, and BIP-VI uses a prior supported on selectors with at most $p_{\max} =20$ active features---below the smallest per-environment sample size considered ($n_e=50$). 

\subsubsection{Synthetic study: empirical verification  of theory}
\label{subsec:synthetic-theory}

\begin{figure}[!t]
    \centering
    \includegraphics[scale=0.52]{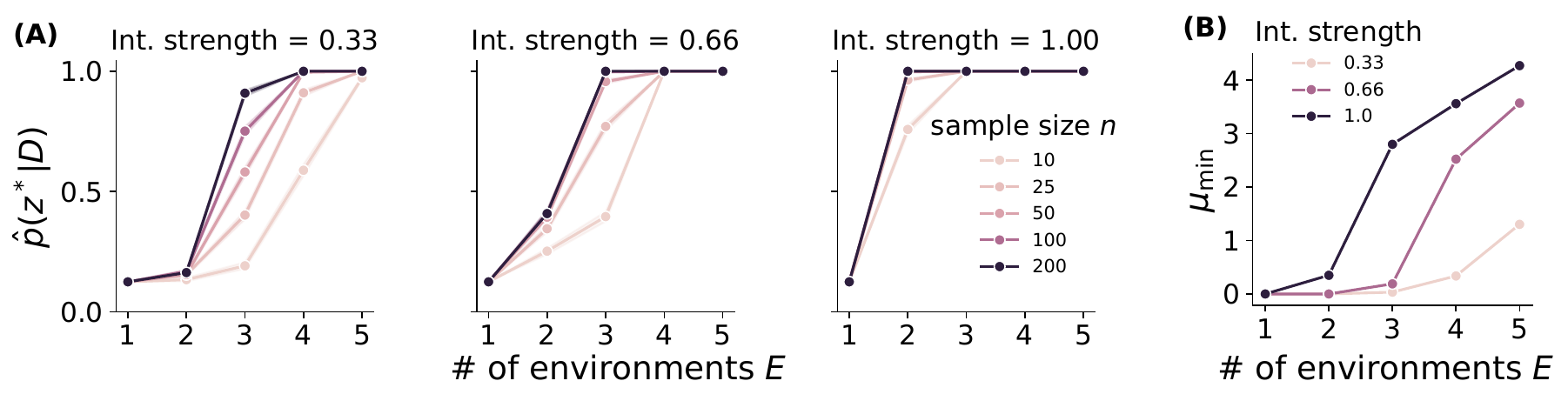}
    \caption{Synthetic study: empirical verification of posterior consistency (left) and contraction rate (right) with $p=3$ features. \textbf{(A)}: The posterior value at the invariant $z^*$, $\hat p(z^* \mid \calD)$, increases with the per-environment sample size $n$, number of environments $E$, and intervention strength ``Int. strength".  Averaged over 1,000 simulations; 95\% confidence bands.  \textbf{(B)}: The heterogeneity measure $\mu_{\min}$ (\Cref{eq:def-mu-min}), which controls the posterior contraction rate $\epsilon_{n,E} = O(R \cdot e^{-\kappa nE \mu_{\min} })$, increases with both the number of environments $E$ as well as intervention strength.  }
    \label{fig:synthetic-theory}
\end{figure}

We study the empirical behavior of the posterior in simulations with
$p=3$ features. We vary per-environment sample size $n$ from 10 to
200, the number of environments $E$ from 1 to 5, and intervention
strength from 0.33 to 1 (intervening on 1 to 3 features).

\Cref{fig:synthetic-theory} (A) shows the posterior value at the true $z^*$, $\hat p (z^* | \calD)$. For each level of intervention strength, $\hat p (z^* | \calD)$ converges to 1 with increasing $n$ and $E$, corroborating the theory on posterior consistency in \Cref{subsec:theory-main}. The empirical rate at which $\hat p (z^* | \calD)$ converges to 1 is faster under higher intervention strength. 
\Cref{fig:synthetic-theory} (B) plots the heterogeneity metric $\mu_{\min}$ that governs the posterior contraction rate $\epsilon_{n,E} = O(R \cdot e^{-\kappa nE \mu_{\min} })$ against the number of environments $E$ under different intervention strength.  We observe that $\mu_{\min}$ increases with both $E$ and intervention strength. In both cases, the observed environments become more heterogeneous, leading to faster posterior contraction.  The trend in \Cref{fig:synthetic-theory} (B) aligns with the empirical trend in \Cref{fig:synthetic-theory} (A), confirming the expected posterior contraction behavior.

\subsubsection{Synthetic study: comparison to other methods}\label{subsec:synthetic-comparison}

We now compare BIP and BIP-VI to other methods in low and high-dimensional settings. 

\parhead{Setup.} We perform separate simulations with feature sizes $p=10$ and $p=450$. The size of true invariant features is capped by 10 in both cases. We vary the per-environment sample size $n$ from $\{50, 200, 500\}$, and vary the number of environments $E$ from 2 to 20 for $p=10$, and from 2 to 40 for $p=450$. For each combination of $(p,n,E)$, we generate 400 datasets under different intervention strengths.

We compare the invariance methods described in \Cref{subsec:comparison-methods}. We run ICP at $\alpha=0.05$; Hidden-ICP at $\alpha=0.05$; EILLS with $\gamma=36$ (default in the original codebase); we return the posterior mode of invariant features from BIP and BIP-VI, to compare against the ground-truth invariant features and to verify the consistency and contraction limits (\Cref{thm:posterior-consistency,thm:rate}) as the data grow. We additionally consider two baselines:
 (1) \textit{Oracle}: pooled linear regression on the true invariant features $x^{z^*}$, returning features significant at $\alpha=0.05$; (2) \textit{Regression}: similar to Oracle but instead regressing on all features; if $p+1>nE$, using Lasso with 10-fold CV. 

\parhead{Evaluation metrics.} Each method produces an estimate of the invariant selector $\hat{z}^*$. Given the ground truth $z^*$, we use two metrics: (i) \textit{exact discovery rate}, which measures the average occurrence where the estimated invariant features exactly recover the true invariant features, and (ii) \textit{coverage}, which measures the average occurrence where the estimated invariant features is a subset of the true invariant features.  A perfect method should achieve 1 on both metrics. We emphasize that coverage alone can be misleading, as always predicting the empty set yields full coverage but zero exact discovery.

\setcounter{topnumber}{1}
\begin{figure}[t]
    \centering
    \includegraphics[scale=0.47]{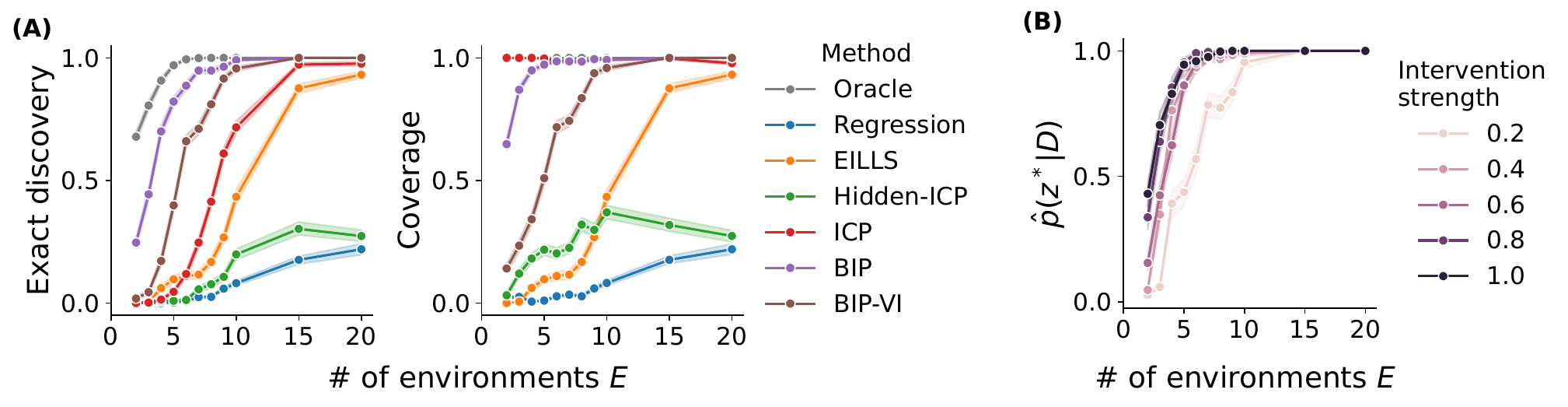}
    \caption{Synthetic study: comparison to other methods with $p=10$ features.  Results are averaged over  cases with $n=50, 200, 500$, each with 400 random simulations under varying intervention strength. Error bars indicate 95\% confidence intervals.
    \textbf{(A):} Exact discovery rate improves with $E$ for all methods, with Oracle converging fastest, followed by BIP and BIP-VI. Oracle and ICP maintain coverage near 1 across all $E$, while coverage for BIP, BIP-VI, and EILLS gradually increases to 1. Other methods show relatively low coverage.
    \textbf{(B):} BIP's posterior value at $z^*$  increases with  both $E$ and  intervention strength, as each case leads to  higher  heterogeneity across environments.}
    \label{fig:synthetic-sweep-p10}
\end{figure}

\parhead{Low-dimensional results.}
\Cref{fig:synthetic-sweep-p10} displays the results for $p=10$. As $E$ increases, all methods improve in exact discovery rate (Panel A, left); at $E=20$ all except Regression and Hidden-ICP reach nearly 1, with Oracle converging fastest, followed by BIP and BIP-VI. 
Coverage (Panel A, right) rises with $E$ for all methods. Oracle and ICP stay near 1 throughout, though at small $E$ ICP does so vacuously, returning an empty prediction in most simulations (high coverage, low exact discovery). BIP and BIP-VI reach near-perfect coverage by $E=20$. 
\Cref{fig:synthetic-sweep-p10}(B) shows BIP's posterior value at the true $z^*$, which converges to 1 as $E$ increases, faster under higher intervention strength, consistent with the theoretical contraction rates. BIP-VI's variational posterior shows the same trend but converges more slowly (\Cref{app:subfig:p10-posterior-vi}).

\begin{figure}[t]
    \centering
    \includegraphics[scale=.47]{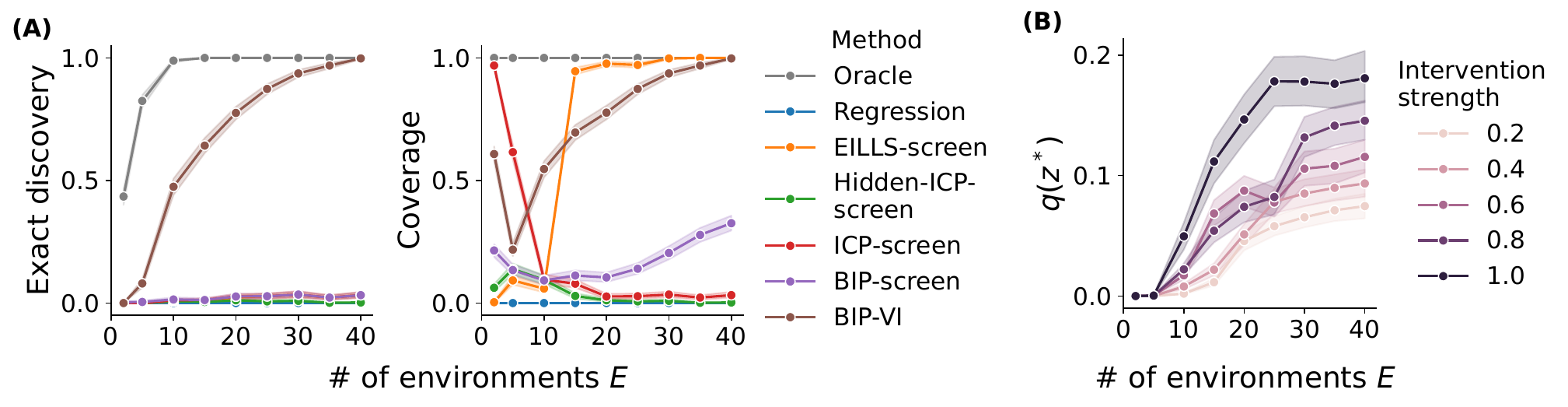}
    \caption{Synthetic study: comparison to other methods with $p=450$ features where $E$ sweeps from $2$ to $40$.  Results are averaged over  cases with $n=50, 200, 500$, each with 400 random simulations under varying intervention strength. Error bars indicate 95\% confidence intervals. Methods with the ``-screen" suffix use a pre-screening step.
    \textbf{(A):} The exact discovery rate of Oracle and BIP-VI converges to 1 as $E$ increases, while it remains low for other methods.  Oracle maintains near-perfect coverage across all $E$, and coverage of BIP-VI gradually approaches to 1. EILLS-screen achieves high coverage at large $E$ but this is  due to frequent empty predictions. Other methods have low coverage.  \textbf{(B):} The variational posterior value at the true $z^*$ by BIP-VI increases with both $E$ and  intervention strength, but remains below 0.2 even at the largest $E=40$ considered.}
    \label{fig:synthetic-sweep-p450}
\end{figure}
\setcounter{topnumber}{2}

\parhead{High-dimensional results.} \Cref{fig:synthetic-sweep-p450} gives the results for $p=450$. Panel A (Left) shows that the exact discovery rate of Oracle and BIP-VI converges to 1 as the number of environments $E$ increases. In contrast, the remaining methods exhibit a near-zero exact discovery rate across $E$, primarily due to failures in the pre-screening step. 
Coverage (Panel A, right) should be read together with exact discovery, since it rewards conservative (empty) predictions. Oracle stays near-perfect, while Regression is low throughout, selecting predictive but non-invariant features. BIP-VI predicts mostly empty sets at $E=2$ (vacuously high coverage), then dips as its predictions become non-empty but imperfect, and approaches 1 as $E$ grows. EILLS-screen's high coverage at large $E$ likewise comes only from empty predictions. At small $E$, ICP also predicts empty sets and scores high coverage, but at large $E$ its coverage falls toward 0: the screened pool increasingly retains no subset passing the invariance test, so ICP returns no prediction (scored as coverage 0). Hidden-ICP-screen is low throughout, and BIP-screen slightly better.

\Cref{fig:synthetic-sweep-p450} (B) plots the variational posterior density $q(z^*)$ at the true $z^*$ by BIP-VI. $q(z^*)$ increases with larger $E$ or higher intervention strength. However, $q(z^*)$ stays well below 1 even at the largest $E=40$, a substantial uncertainty that the mean-field factorization compounds across features: if each marginal $q({z^*}^{(i)}) =.996, \forall i$, the joint value is exponentially smaller, $q(z^*) = .996^{450} = .165$.


%% file: figures/fig_synthetic_dgp.tex

\usetikzlibrary{decorations.pathreplacing, calc}

\definecolor{invcolor}{HTML}{C04828}
\definecolor{invfill}{HTML}{FAECE7}
\definecolor{noncolor}{HTML}{B0AFA8}
\definecolor{nonfill}{HTML}{F1EFE8}
\definecolor{ycolor}{HTML}{378ADD}
\definecolor{yfill}{HTML}{E6F1FB}
\definecolor{ytext}{HTML}{185FA5}
\definecolor{hammercolor}{HTML}{444441}

\newcommand{\hammerpath}{%
    (5.5, -2.5) -- (6, 3)
    -- (3.5, 4.5) -- (0, 5.5)
    .. controls (-2, 7) and (-5, 7) .. (-6.5, 5)
    -- (-5.5, 4) -- (-4, 3) -- (-5, 1.5)
    -- (-5.5, -2.5)
    -- (-1.8, -3.5) -- (-2, -19)
    .. controls (-2, -21) and (2, -21) .. (2, -19)
    -- (1.8, -3.5) -- (4, -2.5) -- cycle%
}

\newsavebox{\hammersym}
\savebox{\hammersym}{\tikz[scale=0.0132, rotate=45]{%
    \fill[hammercolor, opacity=0.8] \hammerpath;
}}

\newsavebox{\hammerleg}
\savebox{\hammerleg}{\tikz[scale=0.008, rotate=45]{%
    \fill[hammercolor, opacity=0.8] \hammerpath;
}}

\newsavebox{\hammericon}
\savebox{\hammericon}{\raisebox{1pt}{\tikz[scale=0.0144, rotate=45]{%
    \fill[hammercolor] \hammerpath;
}}}

\newcommand{\legswatch}[3]{%
    \tikz[baseline=-0.5ex]{%
        \fill[#1] (0,0) rectangle (0.3,0.22);
        \draw[#2, line width=#3] (0,0) rectangle (0.3,0.22);
    }%
}

\begin{figure}[t]
\centering
\begin{tikzpicture}[
    xscale=1.4,
    box/.style={rectangle, rounded corners=3pt, minimum width=1.1cm, minimum height=0.65cm, draw, font=\small\itshape},
    invbox/.style={box, fill=invfill, draw=invcolor, text=invcolor, line width=0.4pt},
    nonbox/.style={box, fill=nonfill, draw=noncolor, text=noncolor!70!black, line width=0.4pt},
    ybox/.style={box, fill=yfill, draw=ycolor, text=ytext, minimum height=0.75cm, line width=0.8pt},
    brace/.style={decorate, decoration={brace, amplitude=4pt, mirror}, gray!70},
    env label/.style={font=\small, anchor=east},
]

\draw[->, gray!60, thick] (0, 4.8) -- (6.8, 4.8);
\node[font=\scriptsize, gray!60, above] at (3.4, 4.8) {Data-generating order};

\foreach \i/\lab in {1/$t^{(1)}$, 2/$t^{(2)}$, 3/$t^{(3)}$, 4/$t^{(4)}$, 5/$t^{(5)}$, 6/$t^{(6)}$} {
    \node[font=\scriptsize\itshape, gray!70] at (\i - 0.05, 4.35) {\lab};
}

\node[env label] at (-0.3, 3.65) {\textbf{$e=1$}};
\node[env label, font=\small\itshape] at (-0.3, 3.3) {(observational)};

\node[invbox]  (o1) at (1, 3.5) {$x^{(2)}$};
\node[invbox]  (o2) at (2, 3.5) {$x^{(4)}$};
\node[ybox]    (o3) at (3, 3.5) {$y$};
\node[nonbox]  (o4) at (4, 3.5) {$x^{(1)}$};
\node[nonbox]  (o5) at (5, 3.5) {$x^{(3)}$};
\node[nonbox]  (o6) at (6, 3.5) {$x^{(5)}$};

\draw[brace] (0.55, 3.05) -- (2.45, 3.05)
    node[midway, below=5pt, font=\scriptsize] {Invariant features $x^{z^*}$};
\draw[brace] (3.55, 3.05) -- (6.45, 3.05)
    node[midway, below=5pt, font=\scriptsize] {Non-invariant features};

\node[env label] at (-0.3, 1.65) {\textbf{$e=2$}};
\node[env label, font=\small\itshape] at (-0.3, 1.3) {(interventional)};

\node[invbox]  (i1) at (1, 1.5) {$x^{(2)}$};
\node[invbox]  (i2) at (2, 1.5) {$x^{(4)}$};
\node[ybox]    (i3) at (3, 1.5) {$y$};
\node[nonbox]  (i4) at (4, 1.5) {$x^{(1)}$};
\node[nonbox]  (i5) at (5, 1.5) {$x^{(3)}$};
\node[nonbox]  (i6) at (6, 1.5) {$x^{(5)}$};

\node at (2.38, 1.90) {\usebox{\hammersym}};
\node at (4.38, 1.90) {\usebox{\hammersym}};
\node at (6.38, 1.90) {\usebox{\hammersym}};

\node[env label] at (-0.3, 0.25) {\textbf{$e=3$}};
\node[env label, font=\small\itshape] at (-0.3, -0.1) {(interventional)};

\node[invbox]  (j1) at (1, 0.1) {$x^{(2)}$};
\node[invbox]  (j2) at (2, 0.1) {$x^{(4)}$};
\node[ybox]    (j3) at (3, 0.1) {$y$};
\node[nonbox]  (j4) at (4, 0.1) {$x^{(1)}$};
\node[nonbox]  (j5) at (5, 0.1) {$x^{(3)}$};
\node[nonbox]  (j6) at (6, 0.1) {$x^{(5)}$};

\node at (1.38, 0.50) {\usebox{\hammersym}};
\node at (5.38, 0.50) {\usebox{\hammersym}};

\node[font=\small, anchor=north] at (3.4, -0.7) {%
    $y$ is never intervened on: in this example, $p(y \mid x^{(2)}, x^{(4)})$ is invariant across all environments};

\end{tikzpicture}

\vspace{2pt}
{\scriptsize 
\hspace{2cm}%
\legswatch{invfill}{invcolor}{0.4pt}\,Invariant feature\quad
\legswatch{yfill}{ycolor}{0.8pt}\,Outcome $y$\quad
\legswatch{nonfill}{noncolor}{0.4pt}\,Non-invariant feature\quad
\raisebox{1pt}{\usebox{\hammerleg}}\,Intervened
}

\caption{Illustration of the synthetic data-generating process with $p=5$ features.
Variables are generated sequentially in a fixed ordering $t^{(1)}, \ldots, t^{(p+1)}$, a permutation of $(x^{(1)}, \cdots, x^{(p)}, y)$.
Features preceding $y$ in the ordering, \textcolor{invcolor}{($x^{(2)}, x^{(4)}$)}, are
invariant: the conditional distribution $p_e(y \mid x^{(2)}, x^{(4)})$ is held fixed
across environments. We first specify the conditional distributions $p_1(t^{(i)} \mid t^{(1:i-1)})$ for $i=1,\cdots, p$ in the observational environment ($e{=}1$). In interventional environments
($e \geq 2$), a randomly selected subset of features (marked with \usebox{\hammericon}) has its conditional modified, while $y$'s  is never changed.}
\label{fig:synthetic-dgp}
\end{figure}

%% file: sections/gene-exp.tex
\subsection{Gene perturbation study}\label{sec:gene-study}

We revisit a real-world application studied in \citet{peters2016causal,meinshausen2016methods,rothenhausler2019causal} which involves predicting genes whose deletion significantly alters the expression of some target gene. All experimental details are provided in \Cref{app:sec:gene-study}. 
Our findings are summarized below.

\begin{itemize}[leftmargin=*]
\item BIP-VI achieves high precision and moderate recall,
  outperforming existing invariant prediction approaches that depend
  on feature pre-screening in high dimensions.

  \item BIP-VI's predictions closely align with the top 10 stable
    findings reported in \citet{meinshausen2016methods}.
\end{itemize}

\parhead{Dataset and setup.} The data come from a large-scale yeast gene-perturbation experiment measuring genome-wide mRNA expression for 6,170 genes \citep{kemmeren2014large}. 
There are two environments: 262 observational samples ($e=1$, no deletion) and 1,479 interventional samples ($e=2$, each a unique gene deletion). We merge all deletion samples into one interventional environment: treating each deletion as its own environment would leave only a single sample per environment, too few to estimate that environment's distribution.
Our goal is to predict whether an unseen gene deletion significantly changes the expression of some target gene. 

Following \citet{peters2016causal}, this problem can be approached by inferring invariant features. We treat the expression level of the target gene as  outcome $y$ and the expression levels of the remaining $p=6,169$ genes as  features $x$. Under the invariance assumption, if an invariant feature  is deleted, then the outcome $y$ is expected to change in the resulting new environment. \citet{peters2016causal} also provides a causal perspective on the invariance solution, where the invariant features are viewed as the direct causes of the target gene. 

We split the $e=2$ interventional samples into 5 blocks and sweep through all 6,170 genes, treating each as the target and the rest as features. For each target, the training set uses all $e=1$ observational samples as one environment and 4 of the 5 interventional blocks as the second, holding out the fifth for validation; we drop any training sample that perturbs the target gene itself, so the outcome is never directly intervened on. Repeating over the five held-out blocks validates each perturbation at least once.

\begin{figure}[!t]
    \centering
    \includegraphics[width=\textwidth]{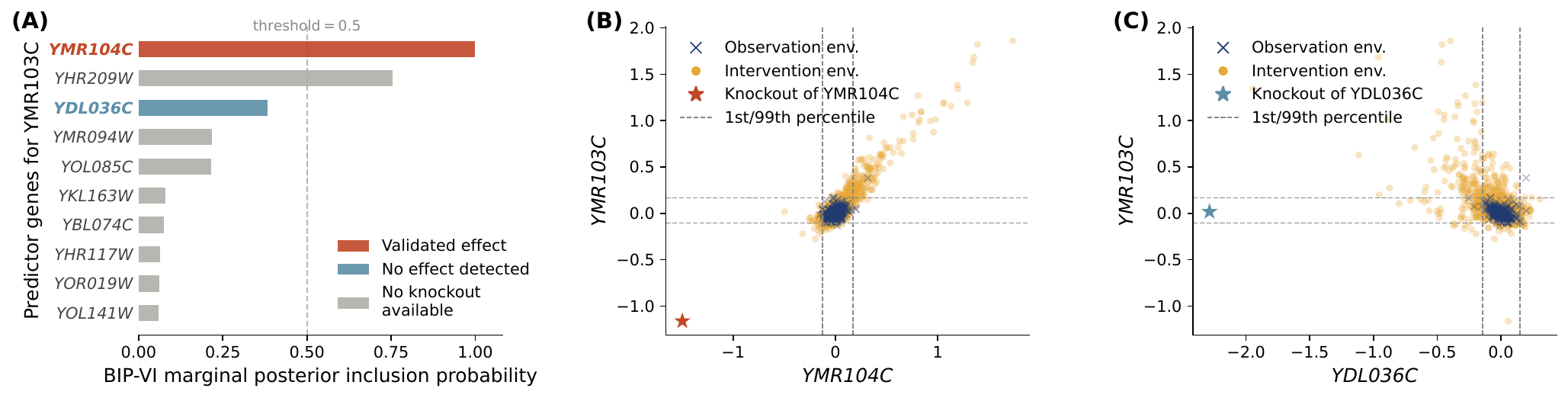}
    \caption{Gene perturbation study: illustration of BIP-VI output and validation procedure on the target gene, YMR103C. \textbf{(A)}: Top 10 feature genes ranked by BIP-VI's marginal posterior inclusion probability; the dashed line marks a posterior inclusion probability threshold of $0.5$ (up to the user's choice) used to declare a prediction. Bars are colored by validation outcome on held-out interventional samples: red if knocking out the feature produces a significant change in YMR103C (validated effect), blue if the knockout produces no detectable change (no effect detected), and gray if no knockout sample for that feature is available in the dataset. \textbf{(B)}: Scatter of the top-ranked feature YMR104C by BIP-VI against the target YMR103C across observational ($\times$) and interventional ($\bullet$) samples; the star marks the YMR104C-knockout sample in the held-out validation. Dashed lines indicate the 1st/99th percentiles in observational data. The knockout sample falls well outside this percentile band, validating YMR104C as a predictor of the target YMR103C. \textbf{(C)}: Same plot for YDL036C. The YDL036C-knockout sample lies far from the bulk along the $x$-axis but the target YMR103C's expression level remains within the percentile band, so knocking out YDL036C has no detectable effect.}
    \label{fig:gene-illustration}
\end{figure}

We apply the invariance methods of \Cref{subsec:comparison-methods}, plus two baselines: marginal regression (features with significant marginal correlation with the outcome) and Lasso with 5-fold cross-validation. The screening-based methods use the top 10 features from the same 5-fold Lasso; BIP-screen places a uniform prior over them, while BIP-VI uses a prior over selectors with at most 200 active features. Since the gene study has no ground-truth invariant set and each held-out knockout validates one feature at a time, we summarize BIP-VI marginally: we report the posterior inclusion probability $q(\di{z}{i}=1)$ and predict genes above a threshold $t$. \Cref{fig:gene-illustration}(A) shows this posterior profile for the target gene YMR103C; BIP-screen is scored the same way over its screened features. 

\parhead{Evaluation metrics.}  We evaluate the predicted invariant gene using a held-out validation set. If deleting the feature gene leads to a significant change in the expression level of the target,
the prediction is considered correct. 
Due to the limited  validation data, not all features can be
empirically validated. \Cref{fig:gene-illustration}(B) and (C) illustrate this validation step for two features of the target YMR103C from BIP-VI's output in Panel~(A). In~(B), knocking out the top-ranked feature YMR104C pushes YMR103C outside its  percentile band in the observational environment, so the prediction is validated as a correct effect. In~(C), knocking out YDL036C (whose posterior inclusion probability is below the $0.5$ threshold) leaves YMR103C within the band, so no effect is detected. Other genes shown in Panel~(A) have no knockout sample available and cannot be validated.

We compute \textit{precision} as the ratio of correct predictions to the total number of predictions that can be validated.  Similarly,  \textit{recall} is the ratio of the correct predictions to the total number of significant gene deletion effects that can be validated. 
\begin{figure}[!t]
    \centering
    \includegraphics[scale=0.55]{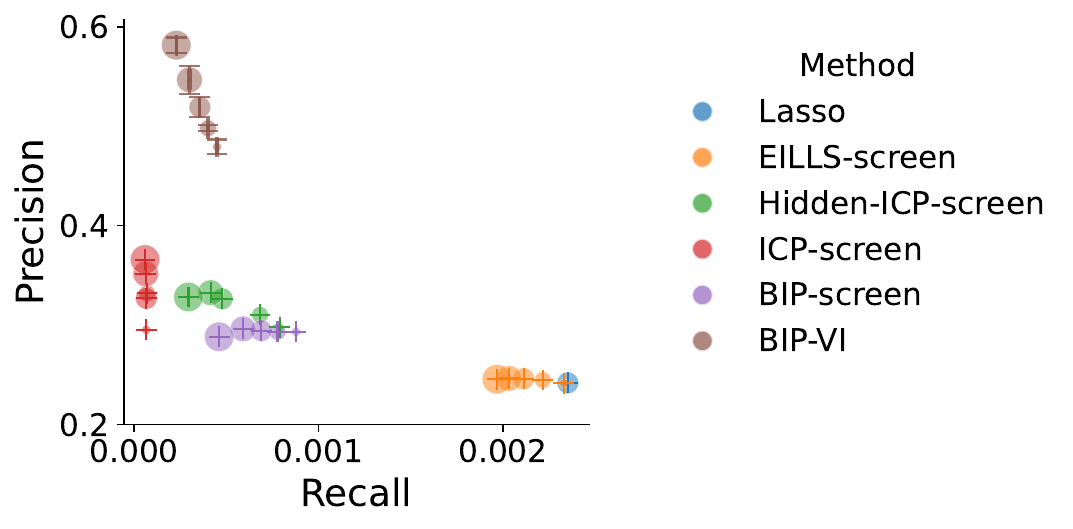}
    \caption{Gene perturbation study: precision v.s. recall. Each dot shows the average result over 3 random seeds, with error bars indicating 2 standard errors. Colors denote methods, and multiple dots per method reflect different hyperparameter settings with larger dots indicating more conservative settings. BIP-VI consistently achieves a favorable precision-recall trade-off with the highest precision and recall higher than ICP-screen across different hyperparameter settings.}
    \label{fig:gene-results2}
\end{figure}

\parhead{Results.} In \Cref{fig:gene-results2}, we trace each method's precision--recall trade-off by sweeping its hyperparameter: the marginal posterior inclusion threshold $t \in \{0.5, 0.6,0.7,0.8, 0.9\}$ for BIP-VI and BIP-screen, the significance level $\alpha \in \{0.001, 0.005, 0.01, 0.05, 0.1\}$ for ICP-screen and Hidden-ICP-screen, and the invariance penalty $\gamma \in \{1, 25, 50, 75, 100\}$ for EILLS-screen. 
We summarize findings as follows: (i) BIP-VI has the highest precision and a moderate recall; (ii) ICP-screen
is the second most accurate but also the most conservative as
indicated by the lowest recall; (iii) Other invariant prediction methods and Lasso have lower precision and higher recall compared to BIP-VI; (iv) Marginal regression (not plotted in the figure) has the lowest precision ($0.11$) but the highest recall ($0.87$), identifying 5,500 of the 6,169 features on average.

BIP-VI's recall remains moderate for several reasons. Although it draws candidates from all features rather than a much smaller pre-screened pool, it reports only features with high marginal posterior inclusion probability, and with only two environments and limited data the posterior carries substantial uncertainty, so it selects relatively few features. Moreover, with only two environments the invariant set need not be unique (\Cref{ass:uniquness} may fail), so the posterior spreads over multiple approximately invariant sets rather than a single $z^*$ (\S\ref{subsec:theory-discussion}, \Cref{thm:posterior-consistency-general}); this multi-modality is hard for the mean-field variational family to capture, leading to partial coverage. Other contributing factors include invariance-assumption violations and outcome-model misspecification.

\parhead{Comparison between BIP-VI and ICP-screen.}
Finally, we provide a detailed comparison between BIP-VI and ICP-screen, as they achieve higher precision than other methods. 
We examine the predictions made by BIP-VI with  a marginal probability threshold of $t=0.5$, and by ICP-screen with a significance level of  $\alpha=0.01$. For a fixed target gene, we take the intersection of predicted invariant feature sets across 3 random seeds per method. We find that approximately 42\% (27 out of 67) of correct  predictions  by ICP-screen are also identified by BIP-VI, suggesting some degree of consistency. Moreover, BIP-VI is less conservative than ICP-screen, with a total of 371 correct predictions -- about 5.5 times as many as ICP-screen. 

\setcounter{footnote}{0} 

\ifjrssb\else
\input{figures/gene_study/top10_comparison_table_float}
\fi

We also compare our findings with those of \citet{meinshausen2016methods}. They  use ICP at a significance level of $\alpha=0.01$ after a Lasso pre-screening procedure, followed by an additional stability selection step  \citep{meinshausen2010stability} applied to 50 randomly bootstrapped datasets. In addition, the dataset used in \citet{meinshausen2016methods} only has 160 observational samples, whereas the updated version that we use has 262 observational samples.\footnote{The updated dataset can be downloaded at \url{https://deleteome.holstegelab.nl/}.} \Cref{tab:gene-study-consistency} gives the results.  Most of the verified effects predicted by BIP-VI, ICP-screen, and \citet{meinshausen2016methods} are the same.


%% file: figures/gene_study/top10_comparison_table_float.tex
\begin{table}[!ht]
 \resizebox{\textwidth}{!}{%
\input{figures/gene_study/top10_comparison_table}
}
\caption{Invariant genes predicted by different methods: ICP from previous findings \citep{meinshausen2016methods} (left column); our implementation of ICP-screen (middle column); and BIP-VI (right column). We observe consistency in most predictions on the 10 target genes selected by \citet{meinshausen2016methods}. Genes in blue are validated to have significant effects on the corresponding target gene; genes marked with a superscript $^*$ cannot be checked given existing data. Genes with a superscript $^\times$ can be checked but do not have a significant effect; $\emptyset$ means no invariant features are predicted.}
\label{tab:gene-study-consistency}
\end{table}

%% file: figures/gene_study/top10_comparison_table.tex
\begin{tabular}{l lll}
\hline
 & \multicolumn{3}{c}{Inferred invariant feature gene(s)}          \\ \cline{2-4} 
\multirow{-2}{*}{Target gene} & \multicolumn{1}{l}{\citet{meinshausen2016methods}}                      & \multicolumn{1}{l}{ICP-screen ($\alpha=0.01$)}                     & BIP-VI ($t=0.5$)                                \\ \hline
YMR103C                       & \multicolumn{1}{l}{{\color[HTML]{3166FF} YMR104C}} & \multicolumn{1}{l}{{\color[HTML]{3166FF} YMR104C}} & {{\color[HTML]{3166FF} YMR104C}, YHR209W$^*$} \\ 
YMR321C                       & \multicolumn{1}{l}{{\color[HTML]{3166FF} YPL273W}} & \multicolumn{1}{l}{{\color[HTML]{3166FF} YPL273W}} & {\color[HTML]{3166FF} YPL273W}          \\ 
YCL042W                       & \multicolumn{1}{l}{{\color[HTML]{3166FF} YCL040W}} & \multicolumn{1}{l}{{\color[HTML]{3166FF} YCL040W}} & {\color[HTML]{3166FF} YCL040W}          \\ 
YLL020C                       & \multicolumn{1}{l}{{\color[HTML]{3166FF} YLL019C}} & \multicolumn{1}{l}{{\color[HTML]{3166FF} YLL019C}} & {\color[HTML]{3166FF} YLL019C}          \\ 
YPL240C                       & \multicolumn{1}{l}{{\color[HTML]{3166FF} YMR186W}} & \multicolumn{1}{l}{{\color[HTML]{3166FF} YMR186W}} & {\color[HTML]{3166FF} YMR186W} \\ 
YBR126C                       & \multicolumn{1}{l}{{\color[HTML]{3166FF} YDR074W}} & \multicolumn{1}{l}{$\emptyset$}                               & YGR008C$^*$, YKL035W$^*$, {{\color[HTML]{3166FF}YDR074W}}               \\ 
YMR173W-A                     & \multicolumn{1}{l}{{\color[HTML]{3166FF} YMR173W}} & \multicolumn{1}{l}{{\color[HTML]{3166FF}YMR173W}, YOL100W$^*$}               & {\color[HTML]{3166FF} YMR173W}          \\ 
YGR264C                       & \multicolumn{1}{l}{{\color[HTML]{3166FF}YGR162W}}                        & \multicolumn{1}{l}{$\emptyset$}                               &                              $\emptyset$           \\ 
YJL077C                       & \multicolumn{1}{l}{ {\color[HTML]{3166FF} YOR027W} }                        & \multicolumn{1}{l}{  $\emptyset$}                               & YLL026W$^*$, {\color[HTML]{3166FF}YOR027W}, YFL010W-A$^*$             \\ 
YLR170C                       & \multicolumn{1}{l}{YJL115W$^\times$}                        & \multicolumn{1}{l}{YDR322C-A$^*$, YDR180W$^*$ YJL184W$^*$, YLR438C-A$^*$}                               &         $\emptyset$                                \\ \hline
\end{tabular}

%% file: sections/discussion.tex
\section{Discussion and future works}\label{sec:conclusion}

We present \textit{Bayesian Invariant Prediction (BIP)} to infer the subset of features that invariantly predict the outcome across diverse environments. For high-dimensional settings, we develop \textit{BIP-VI}, a scalable variational algorithm that addresses the exponential complexity of exact inference. We establish theoretical guarantees on posterior consistency and derive posterior contraction rates that depend both on the amount of data and the heterogeneity of the environments. We validate BIP and BIP-VI in both synthetic and real data studies, demonstrating their accuracy, uncertainty quantification, and scalability.

The fundamental assumptions of BIP are that the data generative process is truly invariant, and that the conditional outcome model reflects the invariant prediction mechanism. When these conditions are met, BIP offers a principled way to identify invariant features and make predictions robust to distribution shifts across environments.  In practice, when these assumptions may not hold exactly, BIP identifies the most approximately invariant feature set within our modeling choices. BIP also assumes the environment labels are given; when they are unknown, recent works infer environments from data \citep[e.g.][]{creager2021environment, lin2022zin} which could be incorporated with BIP.

One limitation of BIP is that it relies on fixed estimates of conditional models and does not quantify uncertainty around them; the estimation bias can obscure the posterior inference especially with small sample size. Integrating conditional model estimation into the overall probabilistic framework is an important avenue of further research.

Moreover, while our method and theory apply generally, our experiments only considered linear Gaussian  models. Future work can explore other conditional model classes with optimization techniques such as amortization over $z$. 

Finally, BIP-VI uses a mean-field variational family; using more expressive approximations to improve inference accuracy is a promising  direction. Another direction is to establish theoretical guarantees for BIP-VI, drawing on recent advances in variational inference for discrete variables \citep{sarkar2021random, kunes2026scalable}.

\ifjrssb
\input{figures/gene_study/top10_comparison_table_float}
\fi


%% file: appendix/appendix.tex
\input{appendix/app_theory}

\input{appendix/app_variational_inference}
\input{appendix/app_synthetic_exp}

\input{appendix/app_gene_perturb_exp}
\input{appendix/app_time_profile}

%% file: appendix/app_theory.tex
\section{Additional Theory and Proofs}
\label{app:sec:theory}
In this section we provide additional theoretical results and proofs that are in complement to \Cref{sec:theory}. 

\parhead{Notation.} We write $\forall_p a$ to mean ``for $p(a)$-almost every realization of  $a$'' where $a$ is a random variable $a$ and $p(a)$ is its probability measure.

\subsection{Preliminary lemmas}
We first introduce two lemmas that are helpful for proofs of the main theorems. 

\begin{lemma}\label{lemma:convergence-of-argmax}
    Let $\{X_n\}$ be a sequence of random variables where each $X_n \in \mathbb{R}^p$ and $p$ is a fixed number. Assume $X_n \stackrel{P}{\to} \mu$ for some constant $\mu \in \mathbb{R}^p$, and define 
     $\calI^* := \{i: \mu_{i} = \mu_{\max}  \}$ where $\mu_{\max}$ is value of the largest component(s) of $\mu$ and therefore $\calI^*$ is the set of corresponding indices. Let  $i_n^* := \argmax_i X_n^{(i)}$ denote the index of the maximal component of $X_n$. Then we have 
    \begin{align*}
    P(i_n^* \in \calI^*) \to 1, \textrm{ as } n \to \infty.
    \end{align*}

    As a special case, when $\calI^* = \{i^*: \mu_{i^*} > \mu_j \forall j \neq i^*\}$ is a singleton, we have  
    \begin{align*}
        i_n^* & \stackrel{P}{\to} i^*, \textrm{ as }  n \to \infty. 
    \end{align*}
    \end{lemma}
\begin{proof} 
    By continuous mapping theorem, 
     $\forall j\not \in \calI^*$,  as $n\to\infty$, 
     \begin{align*}  
     \max_{i \in \calI^*} X_n^{(i)}-X_n^{(j)} \stackrel{P}{\to} \mu_{\max} - \mu^{(j)}> 0. 
     \end{align*}
 Hence $P(\max_{i \in \calI^*} X_n^{(i)} - X_n^{(j)} < 0) \to 0$ as $n \to \infty$. 

    Consequently, as $n\to \infty$, 
    \begin{align*}
        P(i_n^* \not \in \calI^*) &= P(\exists j \not \in \calI^*, \max_{i \in \calI^*} X_n^{(i)} < X_n^{(j)})
         \leq \sum_{j \not \in \calI^*} P(\max_{i \in \calI^*}  X_n^{(i)} - X_n^{(j)} < 0 ) \to 0,
    \end{align*}
and therefore
\begin{align*}
    P(i_n^*  \in \calI^*) \to 1. 
\end{align*}
\end{proof}

\begin{lemma}
\label{lemma:convergnece-of-exp-sum}
    Let $\{X_n\}$ be a sequence of random variables such that its sample average $\frac{S_n}{n}$ converges to some constant $\mu$ in probability, where  $S_n := \sum_{i=1}^n X_i$. When $\mu<0$, Then we have 
    \begin{align*}
    \exp\{S_n\} \stackrel{P}{\to}{0}.
    \end{align*}

And when $\mu>0$, we have 
\begin{align*}
    \exp \{S_n\} \conv \infty. 
\end{align*}
\end{lemma}

\begin{proof} 

Suppose $\mu<0$.  $\forall \epsilon >0$, $\forall \delta >0$,  it suffices to find an $N=N(\epsilon, \delta)$ such that for $n > N$, 
\begin{align*}
        P(\exp \{S_n\} < \epsilon )
 = P(\frac{S_n}{n} < \frac{1}{n}\log \epsilon) > 1-\delta. 
 \end{align*}

Since $\frac{S_n}{n} \stackrel{P}{\to} \mu$, for $\epsilon'=-\frac{\mu}{2}$ and $\delta$ given above, $\exists N'=N'(\epsilon', \delta)$ such that for $n > N'$, 
\begin{align*}
    P(\frac{S_n}{n} < \mu + \epsilon') > 1-\delta. 
\end{align*}
We choose $N=\max (N', \frac{1}{\mu + \epsilon'} \log \epsilon)$. For $n > N$, we have 
\begin{align*}
    \frac{1}{n} \log \epsilon > \mu + \epsilon',  
\end{align*}
and therefore 
\begin{align*}
     P(\frac{S_n}{n} < \frac{1}{n}\log \epsilon) & >  P(\frac{S_n}{n} < \mu + \epsilon') > 1-\delta.
\end{align*}

Next, we consider $\mu>0$. $\forall \delta >0, M >0$, it suffices to find an $N=N(M,\delta)$ such that  for $n>N$, 
\begin{align*}
    P(\exp \{S_n\} > M ) &= P(S_n > \log M)  > 1-\delta.
\end{align*}

Since $\frac{S_n}{n} \conv \mu$, for $\epsilon'=\frac{\mu}{2}$ and $\delta$ given above,  $\exists N'=N'(\epsilon', \delta)$ such that 
\begin{align*}
    P(\frac{S_n}{n} > \mu - \epsilon')  > 1-\delta.
\end{align*}

We choose $N = \max(N, \frac{1}{\mu - \epsilon'}  \log M)$. For $n > N,$ we have 
\begin{align*}
    n(\mu - \epsilon') & > \log M
\end{align*}
and therefore 
\begin{align*}
    P(S_n > \log M) & > P(S_n > n(\mu - \epsilon') > 1-\delta. 
\end{align*}
\end{proof}

\begin{lemma}\label{lemma:invariance} Recall  definition of $\mu(z)$ in \Cref{eq:def-mu}:
\begin{align*}
    \mu(z) &= \mathbb{E}_{p(e) p_e(x^z)} \left[ \kl{p_e( y\mid x^z)}{g(y \mid x^z)} \right]. 
\end{align*}
We have 
\begin{align*}
    \mu(z) = 0 & \Leftrightarrow p_e(y \mid x^z) \textrm{ is invariant with respect to } e \quad  \forall_p e \forall  x^z   
\end{align*}
\end{lemma}
\begin{proof}
By the property of KL divergence, we have  
\begin{align*}
    \mu(z) = 0 & \Leftrightarrow \kl{p_e( y\mid x)}{g(y \mid x^z)} = 0  \quad \forall_p e \forall x^z \\
     &\Leftrightarrow p_e(y \mid x^z)= g(y \mid x^z) \quad \forall_p e \forall x^z \\ 
     & \Leftrightarrow p_e( y \mid x^z) \textrm{ is invariant to } e \quad \forall_p e \forall x^z 
\end{align*}

\end{proof}

\subsection{\texorpdfstring{Proof of \Cref{thm:posterior-consistency}}{Proof of Theorem 1}}\label{app:subsec:proof-consistency}

\begin{repthmconsistency}[Restatement of \Cref{thm:posterior-consistency}]
Suppose \Cref{ass:invariance,ass:uniquness,ass:positive-prior,ass:likelihood-consistency} hold, and take the limits $n\to \infty$ and  $E \to \infty$. The BIP posterior satisfies the following properties: 
\begin{itemize}
    \item consistency of the posterior mode: $\hat z_{n,E}: = \argmax_z \hat p(z \mid \calD) \stackrel{P}{\to} z^*$
    \item consistency of the posterior at $z^*$: $\hat p(z^* \mid \calD) \stackrel{P}{\to} 1. $
\end{itemize}

\end{repthmconsistency}

\begin{proof}
The posterior  in \Cref{eq:posterior-estimated-proportionality} can be re-written  as 
\begin{align}
        \hat p(z \mid\calD) &= \frac{ p(z) \exp\{-\hat S_{n,E}(z)\}}{ \sum_z p(z) \exp\{- \hat S_{n,E}(z) \}} \label{eq:posterior-finite}
\end{align}
where $\hat S_{n,E}(z)$ is the log-likelihood ratio given $z$,  
\begin{align}\label{eq:def-hat-S}
     \hat S_{n,E}(z) &:= \sum_{i=1}^n \sum_{e=1}^E  \log \hat p_e(y_{ei} \mid x_{ei}^z) - \log \hat \pooldist (y_{ei} \mid x_{ei}^z). 
     \end{align}

By \Cref{ass:likelihood-consistency}, as $n, E \to \infty$, 
\begin{align*}
 \frac{1}{nE}    \hat S_{n,E}(z) & \stackrel{P}{\to} \mathbb{E}_{p(e) p_e(y \mid x^z)}\left[ \log p_e(y \mid x^z) \right] - \mathbb{E}_{p(e)p_e(y \mid x^z)}\left[ \log g(y \mid x^z) \right]. 
\end{align*}
Recall the definition of $\mu(z)$ in \Cref{eq:def-mu}, the RHS reduces to $\mu(z)$ and hence 
\begin{align}\label{eq:conv-to-muz}
    \frac{1}{nE}  \hat S_{n,E}(z) & \stackrel{P}{\to} \mu(z). 
\end{align}

Applying \Cref{lemma:convergence-of-argmax}, as $n\to\infty$ and $E\to\infty$
\begin{align*}
    \argmax_{z: p(z)>0}  -\frac{1}{nE} \hat S_{n,E}(z) \conv \argmax_{z: p(z)>0} -\mu(z) = z^*, 
\end{align*}
where the last equality follows from \Cref{ass:uniquness,ass:positive-prior} and \Cref{lemma:invariance}.

Consequently, the posterior mode is consistent:
\begin{align*}
    \hat z_{n,E} &:= \argmax_{z} \hat p(z \mid \calD) \\
    &= \argmax_{z:p(z)>0} \log p(z) - \hat S_{n,E}(z)\\
    &= \argmax_{z:p(z)>0} \frac{1}{nE} \log p(z) - \frac{1}{nE} \hat S_{n,E}(z) \\ 
    & = \argmax_{z: p(z) >0 } -\frac{1}{nE} \hat S_{n,E}(z)  \qquad \textrm{  (for sufficiently large $n,E$)} \\
    & \conv z^* \qquad \textrm{ as } n, E \to \infty.
\end{align*}

Next, we prove the posterior consistency at $z^*$. 

By \Cref{ass:positive-prior}, $p(z^*)>0$. Dividing both the numerator and denominator in the RHS of \Cref{eq:posterior-finite} by $p(z^*) \exp \{ -S_{n,E}(z^*)\}$, the posterior can be expressed as 
\begin{align*}
    \hat p(z^* \mid \calD) &= \frac{1}{1 + \sum_{z\neq z^*} p(z)/p(z^*) \cdot \exp \{ \hat S_{n,E}(z^*) - \hat S_{n,E}(z)\}}.
\end{align*}
To show $\hat p(z^* | \calD) \conv 1$, it suffices to show that $\forall_{p} z \neq z^*$, 
\begin{align}\label{eq:intermediate}
    \exp\{ \hat S_{n,E}(z^*) - \hat S_{n,E}(z)\} &\conv 0. 
\end{align}
By \Cref{eq:conv-to-muz}, $\forall z\neq z^*$, 
\begin{align*}
     \frac{1}{nE} \left[ \hat S_{n,E}(z^*) - \hat S_{n,E}(z) \right] &\conv \mu(z^*) - \mu(z) < 0.  
\end{align*}
Applying \Cref{lemma:convergnece-of-exp-sum} to the above inequality, \Cref{eq:intermediate} holds. Hence we conclude the proof. 
\end{proof}

\subsection{\texorpdfstring{Proof of \Cref{thm:rate}}{Proof of Theorem 2}}\label{app:subsec:proof-rate}

\begin{repthmrate}[Restatement of Theorem~\ref{thm:rate}]
Given \Cref{ass:invariance,ass:uniquness,ass:positive-prior,ass:likelihood-consistency,ass:finite-variance}, 
there exists a sequence $\epsilon_{n,E} = O(R \cdot e^{-\kappa nE \mu_{\min} })$ such that 
\begin{align*}
    P\Bigg(\textrm{TV}\Big( \hat p(z \mid \calD), \delta_{z^*}(z) \Big) >  \epsilon_{n,E}  \Bigg)\to 0 
\end{align*}
as $n\to \infty$ and $E \to \infty$, where  
\begin{align}
    R &= \left(\max_{z\neq z^*} \frac{p(z)}{p(z^*)} \right)\cdot | \textrm{supp}(p(z))|, \label{eq:def-R-app} \\
    \mu(z) &:=\mathbb{E}_{p(e)p_e(x^z)} \left[\kl{ p_e(y \mid x^z)}{ \pooldist (y\mid x^z))}\right], \label{eq:def-mu-app} \\ 
    \mu_{\min} &:= \min_{z \neq z^*, z \in \supp{p(z)}} \bar \mu(z) \label{eq:def-mu-min-app}, 
\end{align}
and $\kappa$ is a fixed value in $(0,1)$. 
\end{repthmrate}

\begin{proof}
Based on the expression of the posterior in \Cref{eq:posterior-finite}, we have the following lower bound on the posterior value at $z^*$
\begin{align}
\label{eq:posterior-finite-inequality}
\begin{split}
      \hat p(z^* \mid \calD  ) 
     & \geq \frac{1}{1+ R_{\max} \sum_{z\neq z^*}   \exp\{  \hat S_{n,E}(z^*)  - \hat S_{n,E}(z) \}}.
\end{split}
\end{align}
where $R_{\max}:= \max_{z\neq z^*} \frac{p(z)}{p(z^*)}$.

We decompose $\hat S_{n,E}(z)$ as follows: 
\begin{align}\label{eq:hat-S-decompose}
\hat S_{n,E}(z) =   S_{n,E}(z) + B_{n,E}(z)
\end{align}
where  
\begin{align}\label{eq:def-S}
      S_{n,E} (z) &:=  \sum_{e=1}^E \sum_{i=1}^{n_e} \log   p_e(y_{ei} \mid x_{ei}^z) - \log   \pooldist (y_{ei} \mid x_{ei}^z),
\end{align}
is the log-likelihood ratio under the true model, and 
\begin{align*}
B_{n,E}(z) & :=  \hat S_{n,E}(z) - S_{n,E}(z)  
\end{align*}
is the estimation bias of the   log-likelihood ratios.

\begin{enumerate}[\textbf{Step } 1.]
    \item We first show the concentration behavior of $S_{n,E} (z^*) - S_{n,E}(z)$. 

For any $z\neq z^*$ and any $k>0$, by Chebyshev's inequality we have 
\begin{align}
    P \bigg( \mid  S_{n,E}(z^*)  - S_{n,E}(z) -\mathbb{E} \left[S_{n,E}(z^*)  - S_{n,E}(z) \right] | \geq k \bigg) & \leq  \frac{ \textrm{Var} \left[S_{n,E}(z^*)  - S_{n,E}(z) \right]}{k^2}.  \label{eq:chebyshev-ineq}
\end{align}

Note that by definitions of $\mu(z)$ and $v(z)$, 
\[ \mathbb{E} \left[S_{n,E}(z) \right] = nE\mu(z) , 
 \qquad    \textrm{Var}\left[ S_{n,E}(z)\right] = nE v(z). \]
Furthermore we have  
\begin{align}\label{eq:ll-mean}
    \mathbb{E} \left[S_{n,E}(z^*)  - S_{n,E}(z) \right] &= \mu(z^*) - \mu(z) < 0 
\end{align}
by \Cref{ass:uniquness}, and 
\begin{align}\label{eq:ll-var}
     \textrm{Var} \left[S_{n,E}(z^*)  - S_{n,E}(z) \right]  \leq 2 nE v_{\max}
\end{align}
where $v_{\max}:=\max_{z} v(z)< \infty$ by \Cref{ass:finite-variance}.

Setting $k = (1-\kappa) nE \mu(z)$ for a fixed value $\kappa \in (0,1)$, and  substituting the expressions for the mean and variance from \Cref{eq:ll-mean} and \Cref{eq:ll-var} into \Cref{eq:chebyshev-ineq}, we obtain  
\begin{align}\label{eq:SnE-concentration}
    P \bigg( \mid  S_{n,E}(z^*)  - S_{n,E}(z) + n  E \mu(z) | \geq (1-\kappa) n E  \mu(z)  \bigg) & \leq  \frac{2v_{\max}}{ n  E (1-\kappa)^2 \mu(z)^2}. 
\end{align}

\item We then show the concentration behavior of $B_{n,E}(z^*) - B_{n,E}(z)$. 

By \Cref{ass:likelihood-consistency} and LLN, as $n \to \infty$ and $E \to\infty$, each estimation bias $\hat B_{n,E}(z)$ goes to 0, and so 
\begin{align}\label{eq:BnE-concentration}
    \frac{1}{nE}\left[ B_{n,E}(z^*) - B_{n,E}(z) \right] & \conv 0. 
\end{align}

\item Having characterized the concentration behaviors of $S_{n,E}(z)$ and $B_{n,E}(z)$, we are now ready to prove the posterior contraction rate. 

We define the rate $\epsilon_{n,E}$  as follows   
\begin{align*}
     \epsilon_{n,E} &:=  1- \frac{1}{1+ R  e^{ -  \kappa nE   \mu_{\min}}}  
     = O(R\cdot e^{-\kappa nE   \mu_{\min}}).  
\end{align*}

Our goal is  to show
\begin{align*} 
P \left(\textrm{TV}\Big( p(z \mid \calD), \delta_{z^*}(z) \Big)  > \epsilon_{n,E} \right) &   \to 0, \qquad \textrm{as } n \to \infty \textrm{ and } E \to \infty.  
\end{align*}

Recall $\textrm{TV}\Big( p(z \mid \calD), \delta_{z^*}(z) \Big) = 1- p(z^* | \calD)$, if suffices to show that 
\begin{align*}
    P \Bigg(  \hat p\bigg(z^* \mid \calD \bigg) < 1-\epsilon_{n,E} \Bigg) & \conv 0, \qquad \textrm{ as } n\to\infty \textrm{ and } E \to \infty. 
\end{align*}

Using the inequality that lower bounds $p(z^* \mid \calD)$ from \Cref{eq:posterior-finite-inequality}, we have 
\begin{align}\label{eq:concentration-inequality-decompose}
\begin{split}
  &P \Bigg(  \hat p\bigg(z^* \mid \calD \bigg) < 1-\epsilon_{n,E} \Bigg) \\
  &\leq P\Bigg(  \frac{1}{1+ R_{\max} \sum_{z\neq z^*}   \exp\{ \hat S_{n,E}(z^*)  - \hat S_{n,E}(z)\}} 
    < 1-\epsilon_{n,E} \Bigg)  \\ 
   &= P  \Bigg(  \frac{1}{1+ R_{\max} \sum_{z\neq z^*}   \exp\{ S_{n,E}(z^*)  - S_{n,E}(z)\}} < 1-\epsilon_{n,E}  
 \mid \mathcal{A}_{n,E}^c \Bigg) P(\mathcal{A}_{n,E}^c)  \\ 
 &\qquad +   P  \Bigg(  \frac{1}{1+ R_{\max} \sum_{z\neq z^*}   \exp\{ S_{n,E}(z^*)  - S_{n,E}(z)\}} < 1-\epsilon_{n,E}  
 \mid \mathcal{A}_{n,E} \Bigg) P(\mathcal{A}_{n,E})
 \end{split}
 \end{align}

where $\calA_{n,E}$ is the event that $ \exists z \neq z^* \textrm{ such that } | S_{n,E}(z^*)  - S_{n,E}(z) + n E \mu(z)| \geq \frac{1}{2} (1-\kappa) n E \mu (z)$ or  $|B_{n,E}(z^*) - B_{n,E}(z) | \geq \frac{1}{2} (1-\kappa) n E \mu (z)$, and $\mathcal{A}_{n,E}^c$ denotes the complement of $\mathcal{A}_{n,E}$.

We consider the following two cases:
\begin{enumerate}
\item Given the event $\calA_{n,E}$, 
by definition of $\calA_{n,E}$ and the subadditivity of probability measure, we have  
\begin{align*}
    P(\calA_{n,E}) & \leq \sum_{z \neq z^*} \Big\{ P\left(|S_{n,E}(z^*)  - S_{n,E}(z) + n E \mu(z)| \geq \frac{1}{2} (1-\kappa) n E \mu (z)|\right) \\
    & \qquad \qquad  + P\left(|B_{n,E}(z^*) - B_{n,E}(z) | \geq \frac{1}{2} (1-\kappa) n E \mu (z) \right) \Big\}. 
\end{align*}
As $n \to \infty$ and $E \to \infty$, the first term in the summation on the RHS of the inequality converges to zero 
 by \Cref{eq:SnE-concentration} from Step 1, and the term converges to 0 as well by \Cref{eq:BnE-concentration} from Step 2. 
Thus, 
\begin{align}\label{eq:AnE-conv}
    P(\calA_{n,E}) \to 0, \qquad \textrm{as } n \to \infty \textrm{ and } E \to \infty. 
\end{align}

    \item Given the complement event $\calA_{n,E}^c$,  $\forall z \neq z^*$ both  $| S_{n,E}(z^*)  - S_{n,E}(z) + n E \mu(z)| < \frac{1}{2} (1-\kappa) n E \mu (z)$, and $|B_{n,E}(z^*) - B_{n,E}(z) | < \frac{1}{2} (1-\kappa) n E \mu (z)$. By the triangle inequality and the decomposition of $\hat S_{n,E}(z)$ in \Cref{eq:hat-S-decompose}, we have 
    \begin{align*}
       & |\hat S_{n,E}(z^*) - \hat S_{n,E}(z^*) + nE\mu(z)| \\&= |S_{n,E}(z^*) - S_{n,E}(z) + nE \mu(z) + B_{n,E}(z^*) - B_{n,E}(z)| < (1-\kappa) nE\mu(z).
    \end{align*}
    Hence, $\calA_{n,E}^c$ implies that  
    \begin{align*}
        \hat S_{n,E}(z^*) - \hat S_{n,E} &< -\kappa nE \mu(z), \qquad \forall z\neq z^* 
    \end{align*}
    which further implies
    \begin{align*}
        \frac{1}{1+ R_{\max} \sum_{z\neq z^*}   \exp\{ S_{n,E}(z^*)  - S_{n,E}(z)\}}  & >  \frac{1}{1+R_{\max} \sum_{z\neq z^*} \exp\{-\kappa nE \mu(z)\}} \\
        & >  \frac{1}{1+R \exp\{-\kappa nE \mu_{\min}\}} =1- \epsilon_{n,E}
    \end{align*}
    where the last inequality follows from the definitions of $R$ and $\mu_{\min}$ in \Cref{eq:def-R-app} and \Cref{eq:def-mu-min-app}, respectively. 

    Therefore, 
    \begin{align}\label{eq:AnE-c-conv}
        P  \Bigg(  \frac{1}{1+ R_{\max} \sum_{z\neq z^*}   \exp\{ S_{n,E}(z^*)  - S_{n,E}(z)\}} < 1-\epsilon_{n,E}  
 \mid \mathcal{A}_{n,E}^c \Bigg)  &=0. 
    \end{align}

\end{enumerate}

Combining the analysis in these two cases, namely \Cref{eq:AnE-conv} and \Cref{eq:AnE-c-conv}, with \Cref{eq:concentration-inequality-decompose} we prove the final contraction rate result: as $n \to \infty$ and $E \to \infty$, 
\begin{align*}
    P \Bigg(  \hat p\bigg(z^* \mid \calD \bigg) < 1-\epsilon_{n,E} \Bigg) & \to 0. 
\end{align*}

\end{enumerate}
\end{proof}

\subsection{\texorpdfstring{Theory extensions to finite $n$ and finite $E$ cases}{Theory extensions to finite n and finite E cases}}\label{app:subsec:theory-extensions}

We present extensions of \Cref{thm:posterior-consistency,thm:rate} to two separate cases: (1) $E$ is held fixed while letting $n \to \infty$, and (2) $n$ is held fixed while letting $E \to\infty$. These results corroborate the first part of \Cref{subsec:theory-discussion}. 

\parhead{Case 1: $E$ fixed and $n\to\infty$.}
Given fixed $E$ environments $\calE_{obs}=\{1,\cdots,E\}$, we set the collection of environments of interests $\calE$ to $\calE_{obs}$. Importantly, the pooled conditional $g(y \mid x^z)$ in \Cref{def:pooled-conditional} and \Cref{ass:invariance},  which depend on $\calE$, should adapt to $\calE_{obs}$ accordingly. We also re-define $p(e) = \frac{1}{E} \sum_{e'=1}^E \delta_{e'}(e)$ as a uniform distribution over the finite set $\calE_{obs}$, and any expectation with respect to $p(e)$  should be re-interpreted. 

We then make the following assumption: 
\begin{assumption}[Estimation consistency given fixed environments]\label{ass:likelihood-consistency-fixed-E}
Given fixed $E$ environments $\calE_{obs}$, $\forall_p z$, as $n\to\infty$, 
\begin{align*}
    \frac{1}{nE}\sum_{i=1}^n \sum_{e=1}^E \log \hat p_e(y_{ei} \mid x_{ei}^z) &  \conv  \mathbb{E}_{p(e)p_e(y \mid x^z)} \left[ \log p_e(y \mid x^z) \right],  \\
    \frac{1}{nE} \sum_{i=1}^n \sum_{e=1}^E \log \hat g(y_{ei} \mid x_{ei}^z) &  \conv \mathbb{E}_{g(y \mid x^z)} \left[ \log g(y \mid x^z) \right].  
\end{align*}
\end{assumption}

\begin{theorem}[Posterior consistency given infinite per-environment samples]\label{thm:posterior-consistency-fiex-E}
Given fixed $E$ environments $\calE$,  and \Cref{ass:invariance,ass:uniquness,ass:positive-prior,ass:likelihood-consistency-fixed-E},  as $n\to \infty$,  we have 
\begin{itemize}
    \item posterior mode consistency, i.e. $\hat z_{n,E}: = \argmax_z \hat p(z \mid \calD) \stackrel{P}{\to} z^*$, and 
    \item posterior consistency at $z^*$, i.e. $\hat p(z^* \mid \calD) \stackrel{P}{\to} 1. $
\end{itemize}
\end{theorem}

\begin{theorem}[Posterior contraction rate given infinite per-environment samples]\label{thm:rate-fixed-E}
Given \Cref{ass:invariance,ass:uniquness,ass:positive-prior,ass:likelihood-consistency-fixed-E,ass:finite-variance}, 
there exists a sequence $\epsilon_{n,E} = O(R \cdot e^{-\kappa nE \mu_{\min} })$ such that 
\begin{align*}
    P\Bigg(\textrm{TV}\Big( p(z \mid \calD), \delta_{z^*}(z) \Big) >  \epsilon_{n,E}  \Bigg)\to 0 
\end{align*}
as $n\to \infty$, where $R$ and $\mu_{\min}$ are defined in \Cref{eq:def-R} and \Cref{eq:def-mu-min}, respectively,  and $\kappa$ is a fixed value in $(0,1)$. 
\end{theorem}

\Cref{ass:likelihood-consistency} is a modification of \Cref{ass:likelihood-consistency}  where we only require the estimation consistency with respect to $n$ within $E$ environments.  

The proofs of \Cref{thm:posterior-consistency-fiex-E,thm:rate-fixed-E} are essentially the same as the proofs of \Cref{thm:posterior-consistency,thm:rate}, respectively, except that $E$ is now held fixed, only taking $n \to\infty$. Hence we abbreviate them here. 

\parhead{Case 2: $n$ fixed and $E \to \infty$.} In this case, we consider a similar setting as in the main theorems where $E$ environments are sampled from $p(e)$ -- a uniform distribution over $\calE$. Given the drawn $E$ environments, we relabel them by $e=1,\cdots,E$. And in each environment $e$, we draw $n$  observations $x_{ei}, y_{ei}$ i.i.d. from $p_e(x,y)$ for $i=1,\cdots,n$.

\begin{assumption}[Estimation consistency of pooled conditional given fixed $n$]\label{ass:consistency-pooled-models-E}  Given fixed $n$, $\forall_p z$, as $E \to \infty$, 
    \begin{align*}
        \frac{1}{n} \frac{1}{E} \sum_{i=1}^n  \sum_{e=1}^E \log \hat \pooldist (y_{ei} \mid x_{ei}^z) & \stackrel{P}{\to} \mathbb{E}_{p(e)p_e(x^z,y)} \left[ \log   \pooldist (y\mid x^z)\right]. 
    \end{align*}
\end{assumption}

\begin{assumption}[Identifiability of $z^*$ under estimation bias]\label{ass:estimation-bias} 
    $z^*$ minimizes $  \mu_n(z)$ defined as follows 
    \begin{align}
        \label{eq:def-bar-mu-n-z}
    \begin{split}
          \mu_n(z)&:=\underbrace{\mathbb{E}_{p(e) p_e(x^z)} \kl{  p_e(y \mid x^z )}{\pooldist (y\mid x^z))}}_{:=\textrm{model discrepancy $   \mu(z) $}} \\
        & +  \underbrace{ \frac{1}{n} \mathbb{E}_{p(e)\prod_{i=1}^n p_e(x_i,y_i)}  \sum_{i=1}^n\left[\log \hat p_e^z(y_i; x_i^z) -  \log   p_e(y_i \mid 
        x_i^z)\right]}_{:=\textrm{estimation bias of per-environment  conditionals }   b_n(z)}, 
        \end{split}
    \end{align}  
    \end{assumption}


\begin{theorem}[Posterior consistency given infinite environments]\label{thm:posterior-consistency-fixed-n}
Given fixed $n$ and \Cref{ass:invariance,ass:uniquness,ass:positive-prior,ass:consistency-pooled-models-E,ass:estimation-bias},  as $E \to \infty$, we have 
\begin{enumerate}
    \item posterior mode consistency, i.e. $\hat z_{n,E} \stackrel{P}{\to} z^*$ 
    \item posterior consistency at $z^*$, i.e. $\hat p(z^* \mid \calD) \stackrel{P}{\to} 1$.
\end{enumerate}
\end{theorem}

\Cref{ass:consistency-pooled-models-E} ensures that the pooled model is consistently estimated given infinite environments, given finite data in each environments. This assumption is reasonable, as we view each observation $(x_{ei}, y_{ei})$ as i.i.d. copy from the pooled distribution $g(x,y):= \int p(e) p_e(x,y)de$.  

\Cref{ass:estimation-bias} is a \textit{non-trivial} assumption -- it requires that the invariant $z^*$ minimizes the combined value of model discrepancy $\mu(z)$ and local model estimation bias $b_n(z)$. Since $z^*$ minimizes $\mu(z)$ by \Cref{ass:uniquness}, \Cref{ass:estimation-bias} is satisfied whenever the estimation bias is well-controlled, for example, if $\max_z |b_n(z)| < \min_{z:z\neq z^*} \mu(z)$. 
In practice, when the per-environment sample size $n$ is sufficiently large,  the estimation bias becomes negligible, so that  \Cref{ass:estimation-bias} is expected to hold.

Incorporating these two assumptions, \Cref{thm:posterior-consistency-fixed-n} establishes posterior consistency in the infinite environment regime. The proof follows the same structure as the proof of \Cref{thm:posterior-consistency}.

\begin{proof} 
$\forall z$, by law of large number (LLN), as $E \to \infty$
\begin{align}\label{eq:lln}
\begin{split}
    \frac{1}{nE} \sum_{i=1}^n \sum_{e=1}^E  \left[ \log \hat p_e (y_{ei} \mid x_{ei}^z) - \log p_e (y_{ei}  \mid x_{ei}^z) \right] & \conv b_n(z), \\    \frac{1}{nE} \sum_{i=1}^n \sum_{e=1}^E \log g(y \mid x^z) & \conv \mathbb{E}_{p(e)p_e(x^z ,y)} \left[ \log g(y \mid x^z) \right], \\
  \frac{1}{nE} \sum_{i=1}^n \sum_{e=1}^E \log p_e (y \mid x^z) & \conv \mathbb{E}_{p(e)p_e(x^z ,y)} \left[ \log p_e(y \mid x^z) \right]. 
\end{split}
\end{align}

By \Cref{ass:consistency-pooled-models-E} and \Cref{eq:lln}, and recall the definitions of  of $\mu(z)$ in \Cref{eq:def-mu} and  $\hat S_{n,E}(z)$ in \Cref{eq:def-hat-S}, we have $\forall_p z$, 
\begin{align}\label{eq:hat-S-conv-finite-n}
 \frac{1}{nE}    \hat S_{n,E}(z)   & \conv \mu_n(z)
\end{align}

Applying \Cref{lemma:convergence-of-argmax}, as $E\to\infty$
\begin{align*}
    \argmax_{z: p(z)>0}  -\frac{1}{nE} \hat S_{n,E}(z) \conv \argmax_{z: p(z)>0} -\mu_n(z) = z^*, 
\end{align*}
where the last equality follows from \Cref{ass:estimation-bias}.

Using the posterior expression in \Cref{eq:posterior-finite}, we show that the posterior mode is consistent:
\begin{align*}
    \hat z_{n,E} &:= \argmax_{z} \hat p(z \mid \calD) \\
    &= \argmax_{z:p(z)>0} \log p(z) - \hat S_{n,E}(z)\\
    &= \argmax_{z:p(z)>0} \frac{1}{nE} \log p(z) - \frac{1}{nE} \hat S_{n,E}(z) \\ 
    &= \argmax_{z: p(z) >0 } -\frac{1}{nE} \hat S_{n,E}(z) \\ 
    & \conv z^*. 
\end{align*}
Next, we prove the posterior consistency at $z^*$. 

By \Cref{ass:positive-prior}, $p(z^*)>0$. Dividing both the numerator and denominator in the RHS of the posterior expression in \Cref{eq:posterior-finite} by $p(z^*) \exp \{ -S_{n,E}(z^*)\}$, the posterior reduces to 
\begin{align}\label{eq:posterior-equiv}
    \hat p(z^* \mid \calD) &= \frac{1}{1 + \sum_{z\neq z^*} p(z)/p(z^*) \cdot \exp \{ \hat S_{n,E}(z^*) - \hat S_{n,E}(z)\}}.
\end{align}
To show $\hat p(z^* | \calD) \conv 1$ as $E \to \infty$, it suffices to show that $\forall_{p} z \neq z^*$, 
\begin{align}\label{eq:intermediate-finite-n}
    \exp\{ \hat S_{n,E}(z^*) - \hat S_{n,E}(z)\} &\conv 0. 
\end{align}
By \Cref{eq:hat-S-conv-finite-n}, $\forall z\neq z^*$, as $E \to \infty$, 
\begin{align*}
     \frac{1}{nE} \left[ \hat S_{n,E}(z^*) - \hat S_{n,E}(z) \right] &\conv \mu_n(z^*) - \mu_n(z) < 0.  
\end{align*}
Applying \Cref{lemma:convergnece-of-exp-sum} to the above inequality, \Cref{eq:intermediate-finite-n} holds. Hence we conclude the proof. 
\end{proof}

Next we establish the contraction rate result. Since the local model estimation bias is non-negligible, we need an additional assumption on its variance: 

\begin{assumption}[Finite variance of estimation bias in per-environment models]\label{ass:finite-variance-estimation-bias}
    For any fixed $n$,  the variance of the estimation bias of per-environment conditionals  
    $$
          v_n (z) := \textrm{Var}_{p(e) \prod_{i=1}^n p_e( y_i \mid x_i^z)} \left[ \sum_{i=1}^n  \left( \log \hat p_e(y_i \mid x_i^z) - \log   p_e(y_i \mid x_i^z)\right) \right]
  $$
    exists and is finite. 
\end{assumption}

\begin{theorem}[Contraction rate given infinite environments]\label{thm:rate-estimated-wrt-E}
Given \Cref{ass:invariance,ass:uniquness,ass:positive-prior,ass:consistency-pooled-models-E,ass:estimation-bias,ass:finite-variance,ass:finite-variance-estimation-bias},  
for any fixed $n$, 
there exists a sequence $\epsilon_{n,E} = O(R  e^{-\kappa nE   \mu_{\min}})$ such that 
\begin{align}\label{eq:def-rate-estimated-moels-infinite-n}
    P\Bigg(\textrm{TV}\Big( p(z \mid \calD), \delta_{z^*}(z) \Big) >  \epsilon_{n,E}  \Bigg)\to 0 
\end{align}
as $E\to \infty$, where  
$R$ and $  \mu_{\min}$ are defined in \Cref{eq:def-R} and \Cref{eq:def-mu-min},  respectively, 
and $\kappa$ is a fixed value in $(0,1)$. 

\end{theorem}

\begin{proof} 
We consider the same posterior lower bound in \Cref{eq:posterior-finite-inequality}, and decompose $\hat S_{n,E}(z)$ as follows 
\begin{align*}
    \hat S_{n,E}(z) &= S_{n,E}(z) + B^{(1)}_{n,E}(z) - B^{(2)}_{n,E}(z), 
\end{align*}
where $S_{n,E}(z)$ is the log-likelihood ratio under the true model in \Cref{eq:def-S}, and 
\begin{align*}
    B^{(1)}_{n,E}(z) &= \sum_{i=1}^n \sum_{e=1}^E \log \hat p_e(y_{ei} \mid x_{ei}^z ) - \log p_e(y_{ei} \mid x_{ei}^z) \\ 
    B^{(2)}_{n,E}(z) &= \sum_{i=1}^n \sum_{e=1}^E \log \hat \pooldist (y_{ei} \mid x_{ei}^z ) - \log \pooldist (y_{ei} \mid x_{ei}^z) 
\end{align*}
are the estimation biases of log-likelihood ratios of local and pooled models, respectively. 

We follow a similar 3-step proof strategy as in the proof of \Cref{thm:rate}. 

\begin{enumerate}[\textbf{Step }1.]
    \item  We first show the concentration behavior of $d_{n,E}(z^*,z):=S_{n,E}(z^*)+B^{(1)}_{n,E}(z^*) - S_{n,E} (z) - B^{(1)}_{n,E}(z)$ with respect to $E$. 

    For any fixed $\kappa \in (0,1)$, by Chebyshev's inequality,   $\forall z\neq z^*$,
    \begin{align}
    \label{eq:chebyshev2}
    P \Big( \Big| d_{n,E}(z^*,z) - \mathbb{E}[d_{n,E}(z^*,z)] \Big| \geq \frac{1-\kappa}{2} \Big| \mathbb{E}[d_{n,E}(z^*,z)] \Big| \Big) 
    & \leq \frac{ \textrm{Var}\left[d_{n,E}(z^*,z)\right]}{ (1-\kappa)^2/4  \cdot \mathbb{E}[d_{n,E}(z^*,z)]^2}.
    \end{align}
    Note that  
    \begin{align}\label{eq:ll-mean-2}
        \mathbb{E}[d_{n,E}(z^*,z)] & = -nE \mu_n(z,z^*)
        \end{align}
    where $\mu_n(z,z^*):=\mu_n(z)-\mu_n(z^*) > 0$ by \Cref{ass:estimation-bias}, and 
    \begin{align}\label{eq:ll-var-2}
        \textrm{Var}[d_{n,E}(z^*,z)] & \leq 2 nE v_{n,\max},
    \end{align}
where $ v_{n,\max} := \max_z v_n(z) + \max_z v(z)$ is finite given \Cref{ass:finite-variance,ass:finite-variance-estimation-bias}.  
    
    Substituting \Cref{eq:ll-mean-2} and \Cref{eq:ll-var-2} into \Cref{eq:chebyshev2}, we have 
    \begin{align}\label{eq:concentration-d}
         P \Big( \Big| d_{n,E}(z^*,z) +nE  \mu_n(z,z^*)  \Big| \geq \frac{1-\kappa}{2} nE\mu_n(z,z^*)  \Big) 
    & \leq \frac{2 v_{n,\max}}{nE (1-\kappa)^2/4  \mu_n(z,z^*)^2}. 
    \end{align}

    \item We then show the concentration behavior of $B^{(2)}_{n,E}(z^*)- B^{(2)}_{n,E(z)}$ with respect to $E$.

    By \Cref{ass:consistency-pooled-models-E} and LLN,  $\forall z$, as $E \to \infty$,
    \begin{align}\label{eq:concentration-B2}
        \frac{1}{nE} B^{(2)}_{n,E}(z^*) - B^{(2)}_{n,E}(z) \conv 0. 
    \end{align}
\item Having characterized the above concentration behaviors, we now prove the contraction rate defined as 
\begin{align*}
\epsilon_{n,E} &:=  1- \frac{1}{1+ R  e^{ -  \kappa nE   \mu_{\min}}}  
     = O(R\cdot e^{-\kappa nE   \mu_{\min}}).  
\end{align*}

We define $ \mathcal{A}_{n,E}$ be the event that $ \exists z \neq z^*$ such that $| d_{n,E}(z^*,z) + nE\mu_{n}(z,z^*)| \geq \frac{1-\kappa}{2} nE \mu_n(z,z^*)$, or $|  B^{(2)}_{n,E}(z^*) - B^{(2)}_{n,E}(z)| \geq \frac{1-\kappa}{2}  nE \mu_n(z,z^*)$.

Following a similar analysis in the proof of \Cref{thm:rate}, we can show that  as $E \to \infty$
\begin{align*}
P\Bigg( \textrm{TV}\Big( \hat p(z \mid \calD), \delta_{z^*}(z) \Big) > \epsilon_{n,E} \Bigg) =  P \Bigg( \hat  p\bigg(z^* \mid \calD \bigg) < 1-\epsilon_{n,E} \Bigg) & \leq P(\mathcal{A}_{n,E}) \to 0
 \end{align*}
where the last convergence follows from \Cref{eq:concentration-d} and \Cref{eq:concentration-B2}. 
\end{enumerate}

\end{proof}

\subsection{Theory generalization  under relaxed assumptions}\label{app:subsec:justificiation-discussion}
 We provide a generalization of \Cref{thm:posterior-consistency} under more general conditions. This result supports the second part of 
 \Cref{subsec:theory-discussion} where some of the original assumptions are violated. 

We define 
the  best-fitting  local model $\bar p_e(y \mid x^z)$ and pooled model $\bar \pooldist (y \mid x^z)$ within the chosen model class $\mathcal{P}_{y \mid x^z}$:
\begin{align*}
    \bar p_e(y \mid x^z)
   &:=
     \argmax_{\tilde p (y \mid x^z) \in \mathcal{P}_{y\mid x^z}} \mathbb{E}_{p_e(x,y)} \log \tilde p (y \mid x^z), \quad \forall e, z \\
   \bar \pooldist(y \mid x^z) &:= \argmax_{\tilde p (y \mid x^z) \in \mathcal{P}_{y \mid x^z} } \sum_{e \in \calE} \mathbb{E}_{p_e(x,y)}\log \tilde p (y_{ei} \mid x_{ei}^z), \quad \forall z,
\end{align*}
and the expected discrepancy between them 
 \begin{align}\label{eq:def-bar-muz}
        \bar \mu(z) &:= \mathbb{E}_{p(e)p_e(x,y)} \left[\log \bar p_e(y \mid x^z) - \log \bar \pooldist (y \mid x^z) \right]. 
    \end{align}

\begin{theorem}[Posterior consistency under general conditions]\label{thm:posterior-consistency-general}
Let $\bar \calZ^*$ be the set of minimizer(s) of $\bar \mu(z)$ defined in \Cref{eq:def-bar-muz}.  Assume that $p(\bar \calZ^*) > 0$, and 
$\forall_p z$, as $n \to \infty$ and $E \to \infty$, 
\begin{align}\label{eq:model-consistency}
\begin{split}
\frac{1}{nE}\sum_{i=1}^n \sum_{e=1}^E \log \hat p_e(y_{ei} \mid x_{ei}^z) &\conv 
\mathbb{E}_{p(e) p_e(y,x^z)}\left[ \log \bar p_e(y \mid x^z) \right], \\
\textrm{and }\frac{1}{nE}\sum_{i=1}^n \sum_{e=1}^E \log \hat g(y_{ei} \mid x_{ei}^z) & \conv 
\mathbb{E}_{p(e)p_e(y , x^z)}\left[ \log \bar g(y \mid x^z) \right].
\end{split}
\end{align}

Then   as $n \to \infty$ and  $E \to \infty$, we have 
\begin{itemize}
    \item $P(\hat z_{n,E} \in \bar \calZ^*) \to 1$, 
    \item $\hat p(\bar \calZ^* \mid \calD) \stackrel{P}{\to} 1. $
\end{itemize}
\end{theorem}

\Cref{thm:posterior-consistency-general} extends the posterior consistency result in \Cref{thm:posterior-consistency}  to a more general setting. 
The positive prior condition is now imposed on the set of minimizers $\bar \calZ^*$ of the  discrepancy measure between the best-fitting local and pooled models. Similarly, the estimation consistency condition in \Cref{eq:model-consistency} is adapted to these models as well. As a result, the posterior concentrates on $\bar \calZ^*$, which can be interpreted as the most \textit{approximately} invariant feature selectors.

\Cref{thm:posterior-consistency} can be viewed as a special instance of \Cref{thm:posterior-consistency-general} under stricter assumptions: the model is correctly specified, and the exact invariant feature selector $z^*$ exists, is unique and has positive prior mass. The proof follows a similar structure as that of \Cref{thm:posterior-consistency}. 

\begin{proof}
We recall the posterior expression 
\begin{align*}
        \hat p(z \mid\calD) &= \frac{ p(z) \exp\{-\hat S_{n,E}(z)\}}{ \sum_z p(z) \exp\{- \hat S_{n,E}(z) \}}
\end{align*}
where $\hat S_{n,E}(z)$ is the log-likelihood ratio given $z$,  
\begin{align*}
     \hat S_{n,E}(z) &:= \sum_{i=1}^n \sum_{e=1}^E  \log \hat p_e(y_{ei} \mid x_{ei}^z) - \log \hat \pooldist (y_{ei} \mid x_{ei}^z). 
\end{align*}

By our assumption, as $n, E \to \infty$, 
\begin{align}\label{eq:conv-S-to-bar-mu}
 \frac{1}{nE}    \hat S_{n,E}(z) & \stackrel{P}{\to} \mathbb{E}_{p(e) p_e(y \mid x^z)}\left[ \log \bar p_e(y \mid x^z) \right] - \mathbb{E}_{p(e)p_e(y \mid x^z)}\left[ \log \bar g(y \mid x^z) \right] = \bar \mu(z).  
\end{align}

For large $nE$, we have 
\begin{align*}
    \hat z_{n,E} &:= \argmax_{z} \hat p(z \mid \calD) \\
    &= \argmax_{z:p(z)>0} \log p(z) - \hat S_{n,E}(z)\\
    &= \argmax_{z:p(z)>0} - \frac{1}{nE} \hat S_{n,E}(z). 
\end{align*}

By \Cref{lemma:convergence-of-argmax} and \Cref{eq:conv-S-to-bar-mu}, as $n\to\infty$ and $E\to\infty$
\begin{align*}
   P(\hat z_{n,E} \in \bar \calZ^* ) & \to 1. 
\end{align*}
where we recall $\bar \calZ^*$ is the set of minimizers of $\bar \mu(z)$. 

Next, we consider any $z \not \in \bar \calZ^*$ with $p(z)>0$. Then   
\begin{align*}
\hat p(z \mid \calD) & \leq  \frac{ p(z) \exp\{-\hat S_{n,E}(z)\}}{p(z) \exp\{-\hat S_{n,E}(z)\} + \sum_{z' \in \bar \calZ^*} p(z') \exp\{- \hat S_{n,E}(z') \}} \\
&= \frac{1}{1  + \sum_{z' \in \bar \calZ^*} p(z') / p(z) \exp\{\hat S_{n,E}(z)- \hat S_{n,E}(z') \}}. 
\end{align*}
Note that by our assumption, $\forall z' \in \bar \calZ^*$, 
\begin{align*}
    \frac{1}{nE} \left(\hat S_{n,E}(z)- \hat S_{n,E}(z') \right) & \conv \bar \mu(z) - \bar \mu(z') >0, 
\end{align*}
and consequently by \Cref{lemma:convergnece-of-exp-sum}, we have $\exp \{ \hat S_{n,E}(z)- \hat S_{n,E}(z')\} \conv \infty$. 

Consequently, $\forall z \not \in \bar \calZ^*$, $\hat p(z \mid \calD) \conv 0$, and therefore $\hat p(\bar \calZ^* \mid \calD) \conv 1$. 
\end{proof}

%% file: appendix/app_variational_inference.tex
\section{Details on BIP-VI}

We provide additional details on BIP-VI, including gradient estimation,  optimization guidelines and implementation considerations. 

\subsection{Gradient estimators for binary latent variables}

Optimizing BIP-VI requires computing gradients of the ELBO in \Cref{eq:elbo}, which involves expectations with respect to the discrete variational distribution $q_\phi(z)$. To obtain an unbiased gradient estimate, we use the U2G estimator proposed by \citet{yin2020probabilistic}. 

 \Cref{alg:u2g} describes the general U2G estimator for computing an unbiased estimate  of $\nabla_\phi \mathbb{E}_{q_\phi(z)} [f(z,\calD, \phi)]$, given an objective function $f$, dataset $\calD$, and variational parameters $\phi$. 
The algorithm constructs a single gradient estimate using pair of correlated samples $z_1$ and $z_2$ to reduce variance while maintaining unbiasedness. Multiple gradient estimates can be obtained by independently repeating this procedure, and their average is used to form the final gradient estimate. 
 
To use U2G in our setting (\Cref{alg:viml}),  the objective function corresponds to the integrand inside the ELBO expression from \Cref{eq:elbo}, that is, 
\begin{align}\label{eq:vi-stochastic-objective}
f(z, \calD, \phi) &:= 
[\log p(z) + \sum_{e=1}^E \sum_{i=1}^{n_e}\log \frac{\hat \pooldist(y_{ei}\mid x_{ei}^z)}{\hat p_e(y_{ei}\mid x_{ei}^z)} - \log q_\phi(z), 
\end{align}
where we replace the true pooled and local conditionals in the ELBO with estimates.

\begin{algorithm}[!t]
\caption{U2G gradient estimation \citep{yin2020probabilistic}}
\label{alg:u2g}
\KwInput{Objective function $f$, dataset $\calD$, variational parameter $\phi$}
\KwOutput{A random unbiased estimate $\hat g$ of $\nabla_\phi \mathbb{E}_{q_\phi (z)} \left[ f(z, \calD, \phi)\right]$}
Draw $u \sim \prod_{i=1}^p \textrm{Uniform}[0,1]$ \label{line:u2g-sample-u} \\
$z_1\leftarrow \mathbf{1} \left[ u> 1-\sigmoid(\phi) \right], z_2\leftarrow \mathbf{1} \left[ u<\sigmoid(\phi) \right]$ \label{line:u2g-sample-z1-z2} \\
$\hat g\leftarrow \frac{1}{2} \sigmoid(|\phi|) \cdot \left[f(z_1, \calD, \phi) - f(z_2, \calD, \phi) \right] \cdot \left( z_1 - z_2\right)$  \label{line:u2g-evaluate}
\end{algorithm}

\subsection{Optimization guidelines}\label{app:vi:optimization}

We provide guidelines for setting and tuning the hyperparameters.  All of these hyperparameters can be tuned by monitoring the training-set ELBO and the overall optimization behavior. In general,  we expect a graduallly increasing ELBO trend during training, with higher ELBO values indicating better model fit. 
We refer to experiment sections for exact configurations used in each setting.

The key hyperparameters are as follows: 
\begin{itemize}
    \item Number of stochastic gradient samples $M$: At each iteration, we estimate the gradient of the ELBO using $M$ samples. Increasing $M$ can reduce the variance but increase the computational cost. In our experiments, we found that $M$ between 10 and 20 balances performance and computational cost.
    \item Prior $p(z)$: The choice of prior is crucial in high-dimensional settings, as it  influences the sparsity of the inferred invariant sets and the initialization of variational parameters (see below). 
    We recommend restricting the prior support to feature subset at size at most $p_{\max}$, which can be informed by the domain knowledge. In practice, we found that the choice of $p_{\max}$ can affect optimization dynamics and can be treated as a tunable hyperparameter.
    \item Initialization of variational parameters $\phi$: The variational parameters $\phi$ are initialized to ensure that the expected number of invariant features is within the prior support. Specifically, we recommend initializing $\phi$ at $\phi_0$ such that 
    \begin{align*}
    \mathbb{E}_{q_{\phi_0}(z)}\left[\|z\|_0\|\right] & \leq \max_{z: p(z)>0} \|z\|_0. 
    \end{align*}
    \item Optimizer: We use stochastic gradient descent (SGD) throughout all experiments, following the practice of  the original U2G paper \citep{yin2020probabilistic}. 
    \item Learning rate scheduling: We use a cyclical learning rate scheduler \citep{smith2017cyclical} to encourage mode exploration. However, alternative schedulers could be explored. 
    \item Number of optimization steps: The  number of optimization steps should be chosen based on the problem scale and convergence behavior. In our experiments, we set a fixed budget for the total number of optimization steps and select the iteration  with the highest training-set ELBO. 
\end{itemize}

\subsection{Implementation details}\label{app:subsec:vi-implementation}
We discuss several algorithmic techniques for \Cref{alg:viml} that offer practical benefits. 

\begin{itemize}
    \item Analytical KL gradients: The ELBO in \Cref{eq:elbo} can be decomposed into an expected likelihood term and a KL regularization term (up to a constant $C$)
    \begin{align*}
        \mathcal{L}(\calD, \phi) 
 &= \mathbb{E}_{q_\phi(z)}\left[ \sum_{e=1}^E \sum_{i=1}^{n_e}\log \frac{\pooldist(y_{ei}\mid x_{ei}^z)}{p_e(y_{ei}\mid x_{ei}^z)} \right] + \kl{q_\phi(z)}{p(z)} + C,
    \end{align*}
    
    Notably,  under a uniform prior, the KL regularization term admits a closed-form gradient expression: 
    \begin{align*}
      \nabla_{\phi} \kl{q_\phi(z)}{p(z)} &=  \left[\log (\sigma) - \log (1-\sigma)\right] \sigma (1-\sigma) \Big|_{\sigma = \frac{1}{1+e^{-\phi}}}
    \end{align*}
    
    To balance  variance reduction with the  implicit generalization behavior of stochastic gradients, we can consider a hybrid strategy: at each iteration, with some probability, we compute the analytical gradient of the KL term; otherwise, we use its stochastic estimate via the U2G estimator in \Cref{alg:u2g}.  The gradient of the  reconstruction term, $\nabla_\phi \mathbb{E}_{q_\phi(z)}\left[ \sum_{e=1}^E \sum_{i=1}^{n_e}\log \frac{\pooldist(y_{ei}\mid x_{ei}^z)}{p_e(y_{ei}\mid x_{ei}^z)} \right]$,  is always estimated by U2G.

    \item Informative prior:  
    In settings where the prior imposes a maximum number of invariant features $p_{\max}$ or restricts the support with in other ways, we design a tractable optimization procedure that enforces thes constraints. 
    
    Specifically, during each U2G estimation step, we draw a pair of correlated samples of $z$ (\Cref{alg:u2g}, \Cref{line:u2g-sample-z1-z2}). Since the variational distribution does not explicitly encode prior constraints, it is possible for some sampled $z$ to fall outside of the prior support, resulting in an infeasible objective value in \Cref{eq:vi-stochastic-objective}, where $\log p(z) = -\infty$. 
    
    To address this issue, we manually assign a fixed, low objective value for infeasible $z$ samples, to guide the optimization to toward high-objective regions within the feasible support. This penalty can be tuned for the specific task, typically by referencing the range of feasible objective values.  In all of our experiments, we set this penalty value to $-1$. 
\end{itemize}



%% file: appendix/app_synthetic_exp.tex
\section{Synthetic data study}\label{app:sec:synthetic-data}
We first describe the general data generative process for the experiments in \Cref{subsec:synthetic-theory,subsec:synthetic-comparison}, and then we provide  details for each experiment in \Cref{app:subsec:synthetic-theory,app:subsec:synthetic-uq,app:subsec:synthetic-comparison}. 

\begin{enumerate}
    \item \textbf{The factorization of the joint distribution.}
    To simulate multi-environment data, we specify a series of joint distributions that factorize consistently across environments:
    \begin{align*}
    p(\di{x}{1:p}, y) =  \prod_{i=1}^{p+1} p(\di{t}{i} \mid \di{t}{1:i-1}), 
    \end{align*}
   where permutation $\pi$ is drawn uniformly over $[1,\cdots, p+1]$, and $\di{t}{i} = \di{x}{\pi(i)}$ if $\pi(i) \neq p+1$ and $\di{t}{i} = y$ otherwise. \footnote{If there are pre-specified lower bound $p^*_{\min}$ and upper bound $p^*_{\max}$ on the number of the true invariant features $\|z^*\|_0$, we keep uniformly sampling $\pi$ until we get a $\pi$ such that $p^*_{\min} < \pi^{-1}(p+1) \leq p^*_{\max} + 1$.} For convenience, we define $\di{t}{1:0}$ to be an empty conditioning set. 

    Each conditional density is a linear Gaussian family, that is, 
    $$p_e(\di{t}{i} \mid \di{t}{1:i-1}) = \mathcal{N}( \di{t}{i} \mid \alpha_{ei}^\top \di{t}{1:i-1}+ \beta_{ei}, \sigma_{ei}^2), $$
    for 
    some parameters $\alpha_{ei} \in \mathbb{R}^{i-1}, \beta_{ei}, \sigma_{ei} \in \mathbb{R}$. For $i=1$ we set $p_e(\di{t}{1} \mid \di{t}{1:0}) = \mathcal{N}( \di{t}{1} \mid \beta_{e1}, \sigma_{e1}^2)$ with parameters $\beta_{e1}, \sigma_{e1} \in \mathbb{R}$.  (For convenience, $\di{t}{1:0}$ denotes an empty vector.)

    We next specify the joint  distributions corresponding to different environments. 

\item \textbf{Observational environment ($e=1$).} For $i = 1,\cdots, p+1$:  
\begin{itemize}
    \item  sample the intercept parameter $\beta_{1i}$ from $\mathcal{N}(0, 1)$,
    
    \item  sample the variance parameter $\sigma_{ei}^2$ from $\uniform{[\sigma^2_{\min}, \sigma^2_{\max}]}$, where $\sigma^2_{\min}$ and $\sigma^2_{\max}$ are lower and upper bounds on the variance parameter, 
    \item sample each dimension of the linear coefficient  parameter $\alpha_{ei}$ independently: 
    \begin{itemize}
        \item  if $\pi(i) = p+1$, i.e. $\di{t}{i} = y$, first sample its absolute value from $\uniform{[lb, ub]}$ and then assign a random sign;  
        \item  if $\pi(i) \neq p+1$ i.e. $\di{t}{i} = \di{x}{\pi(i)}$, (i) with probability $p_{\text{act}} \in [0,1]$ first sample its absolute value from $\uniform{[lb, ub]}$ and then assign a random sign, (ii) otherwise set to 0. 
    \end{itemize}
\end{itemize}

\item \textbf{Interventional Environments ($e=2,\cdots, E$)}  For each interventional environment $e>1$, we first randomly draw a fraction $\rho^{\text{int}}_e$ of \textit{features} $x'_e \subset x$ to be intervened on -- referred to as  the \textit{intervention strength} in the main text. Then, for $i=1,\cdots, p+1$, if $\di{t}{i} \not \in x'_e$, we set $p_e (\di{t}{i} \mid \di{t}{1:i-1}) = p_1(\di{t}{i} \mid \di{t}{1:i-1})$; otherwise, we draw new parameters for $p_e (\di{t}{i} \mid \di{t}{1:i-1})$ as follows 
    \begin{itemize}
        \item sample the intercept parameter $\beta_{ei}$ by first sampling its absolute value from $\mathcal{N}(m_e, 1)$ and assigning to it a random sign, for some $m_e \in \mathbb{R}$  specific to the environment $e$. 
        \item sample the variance parameter $\sigma_{ei}^2$ by first sampling 
        a multiplicative factor $\lambda \sim \textrm{Uniform}([\lambda_{e, \min}, \lambda_{e, \min} + \lambda_{e,\textrm{diff}}])$, and then setting $\sigma_{ei}^2 = \lambda^2 \sigma_{1i}^2 $,   
        where $\lambda_{e, \min}$,  $\lambda_{e,\textrm{diff}}$ are parameters specific to the environment $e$. 
         \item sample each dimension of the linear coefficient  parameter $\alpha_{ei}$ independently: 
    \begin{itemize}
        \item with probability $p_{\text{change}}$, copy the corresponding coefficient  from $e=1$ case
        \item with probability $1-p_{\text{change}}$, (i) with probability $p_{\text{act}}$ sample its absolute value from $\uniform{[lb_e, ub_e]}$, assign a random sign and (ii) otherwise set to 0.  
    \end{itemize}
    \end{itemize}

\item \textbf{Draw data.} Finally, we can draw $n$ data points $\{x_{ei}, y_{ei}\}_{i=1}^n$ from each environment $e$ independently, for $e=1,\cdots, E$. 
\end{enumerate}

\subsection{Empirical verification of theory}\label{app:subsec:synthetic-theory}

We include simulation details for the experiments in \Cref{subsec:synthetic-theory}. We follow the general data generative process described in the preamble of \Cref{app:sec:synthetic-data} with the following configuration: 

We set $p=3$, $p^*_{\max} =3$,  $p^*_{\min} = 1$, and $p_{\text{act}}=1$. 

For $e=1$, we set $lb=0.5$, $ub=2$; $p_{\text{act}}=1.0$, $\sigma_{\min}=0.1$ and $\sigma_{\max} = 0.2$. 
 
 For $e=2,\cdots, E$, we sample $m_e \sim \textrm{Uniform}([0,1])$,  $\lambda_{e,\min} \sim   \textrm{Uniform}([0.1, 0.2])$, and $\lambda_{e, \textrm{diff}} \sim     \textrm{Uniform}([0.1, 0.5])$, and we set $lb_e= (m_e+0.01)^2$, $ub_e=(m_e + 0.5)^2$, and $p_\text{change} = 1$.

\subsection{Identifiability under limited environments}\label{app:subsec:synthetic-uq}

We present three synthetic examples illustrating how the posterior behaves when multiple invariant feature selectors exist under a limited number of environments, supporting the discussion in \Cref{subsec:theory-discussion}. All experiments are simulated with $n=200$ per-environment samples and use $p=2$ features $x = [\di{x}{1}, \di{x}{2}]$.

\parhead{Example 1.}
The joint distributions $\{p_e(x,y)\}_{e=1}^3$ factorize as follows
\begin{align*}
    p_1(x,y) &= p_1(\di{x}{1}) p(y \mid \di{x}{1}) p_1(\di{x}{2} \mid y) \\
    p_2(x,y) &= p_2(\di{x}{1}) p(y \mid \di{x}{1}) p_2(\di{x}{2} \mid y) \\
    p_3(x,y) &= p_3(\di{x}{1}) p(y \mid \di{x}{1}) p_3(\di{x}{2} \mid \di{x}{1}, y)
\end{align*}
where
\begin{align*}
    &p_1(\di{x}{1}) = \mathcal{N}(\di{x}{1} \mid 0, 0.1^2), \, p_2(\di{x}{1}) =  \mathcal{N}(\di{x}{1} \mid 2, 0.1^2),\, p_3(\di{x}{1}) = \mathcal{N}( \di{x}{1} \mid 5, 0.1^2), \\
    &p(y \mid \di{x}{1}) = \mathcal{N}(y \mid \di{x}{1} + 0.5, 0.1^2),  \\
    & p_1(\di{x}{2} \mid y) = p_2(\di{x}{2} \mid y) = \mathcal{N}(\di{x}{2} \mid y + 0.1, 0.1^2), \, \\
    &p_3( \di{x}{2} \mid \di{x}{1}, y) = \mathcal{N}(\di{x}{2} \mid \di{x}{1} + y + 0.1, 0.1^2).
\end{align*}

Considering only the first two environments ($p_1(x,y)$, $p_2(x,y)$), both $\di{x}{1}$ and $\di{x}{1:2}$ are invariant features. As a result, the posterior places mass on both $z = [1, 0]$ and $z = [1, 1]$, as shown in \Cref{fig:synthetic-uq} (left). When the third environment ($p_3(x,y)$) is introduced, only $\di{x}{1}$ remains invariant, and the posterior consequently concentrates on $z = [1, 0]$, as shown in \Cref{fig:synthetic-uq} (right).

\begin{figure}[!t]
    \centering
    \includegraphics[width=0.75\linewidth]{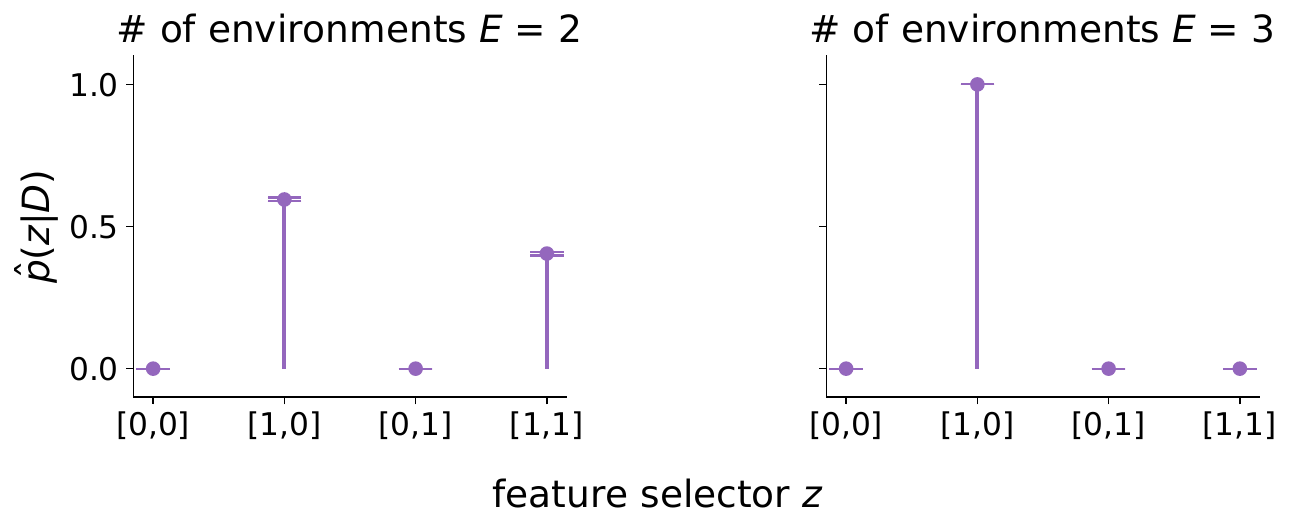}
    \caption{Example 1: posterior behavior under limited environments with $p=2$ features. With $E=2$ environments (left), the posterior $\hat p(z \mid \calD)$ is multi-modal, centered at the two invariant solutions $z=[1,0]$ and $z=[1,1]$. When a third environment is introduced ($E=3$, right), $z=[1,0]$ becomes the only invariant feature selector, and the posterior concentrates accordingly. Averaged over 1{,}000 simulations; 95\% confidence bands.}
    \label{fig:synthetic-uq}
\end{figure}

We include two further examples in which multiple invariant solutions exist under a limited number of environments. In both cases, $E=2$ admits two distinct invariant feature selectors within the first two environments; after introducing a third environment, only one remains identifiable. The posterior tracks this transition in each case.

\parhead{Example 2.} The joint distributions $\{p_e(x,y)\}_{e=1}^3$ are given by 
\begin{align*}
    p_1(x,y)  &= p_1(\di{x}{1}) p_1(\di{x}{2}) p ( y \mid \di{x}{1}, \di{x}{2}) \\ 
    p_2(x,y) & = p_2(\di{x}{1}) p_2(\di{x}{2}) p( y \mid \di{x}{1}, \di{x}{2}) \\
    p_3(x,y) & = p_3(\di{x}{1}) p_3(\di{x}{2}) p( y \mid \di{x}{1}, \di{x}{2})  
\end{align*}
where
\begin{align*}
    &p_1(\di{x}{1}) = \mathcal{N}(\di{x}{1} \mid 0, 0.1^2),\,  p_2(\di{x}{1}) = \mathcal{N}(\di{x}{1} \mid 2, 0.1^2), \, p_3(\di{x}{1}) = \mathcal{N}(\di{x}{1} \mid 5, 0.1^2) \\
    &p_1(\di{x}{2}) = p_2(\di{x}{2}) = \mathcal{N}(\di{x}{2} \mid 0.5, 0.1^2), \quad p_3(\di{x}{2}) = \mathcal{N}(\di{x}{2} \mid -0.5, 0.1^2) \\ 
    & p ( y \mid \di{x}{1}, \di{x}{2}) = \mathcal{N}(y \mid \di{x}{1} + \di{x}{2} + 0.1, 0.1^2).
\end{align*}

$\di{x}{1:2}$ is the only set of invariant features across $\calE = \{1,2,3\}$ with $z^* = [1,1]$  the invariant feature selector. 
But within the first two environments $\calE'=\{1,2\}$, $z=[1,0]$ is also an invariant feature selector: By  \Cref{def:pooled-conditional},  the pooled conditional within $\calE'$ is 
\begin{align*}
    g_{\calE'}  (y \mid \di{x}{1}) &:=  \frac{\sum_{e \in \calE'}\int p(e) p_e(x,y) \diff \di{x}{2}}{\sum_{e\in \calE'} p(e) p_e(\di{x}{1}) }\\  
    &= \frac{\int \frac{1}{2} p_1(\di{x}{1}) p_1(\di{x}{2}) p( y \mid \di{x}{1}, \di{x}{2}) + \frac{1}{2} p_2(\di{x}{1})p_2(\di{x}{2}) p( y \mid \di{x}{1}, \di{x}{2}) \diff \di{x}{2} }{ \frac{1}{2} p_1(\di{x}{1}) + \frac{1}{2} p_1 (\di{x}{2})} \\
    &= \int \frac{1}{2} p_1(\di{x}{2}) p( y \mid \di{x}{1}, \di{x}{2}) + \frac{1}{2} p_2(\di{x}{2}) p( y \mid \di{x}{1}, \di{x}{2}) \diff \di{x}{2}  \\
    &= p_1( y \mid\di{x}{1}) = p_2( y \mid \di{x}{1})., 
\end{align*}
where the last equality follows from $p_1(\di{x}{2}) = p_2(\di{x}{2})$.   This result shows that  $p_e(y \mid \di{x}{1})$  is invariant within environments  $\calE'=\{1,2\}$. However, one can show that it is not the case when we extend to $\calE = \{1,2,3\}$, as $p_3(y \mid \di{x}{1}) \neq p_1 (y \mid \di{x}{1})$. 
The corresponding posterior distributions for the two scenarios are shown in \Cref{fig:synthetic-uq-example2}. 

\begin{figure}[!t]
    \centering
    \includegraphics[width=0.75\linewidth]{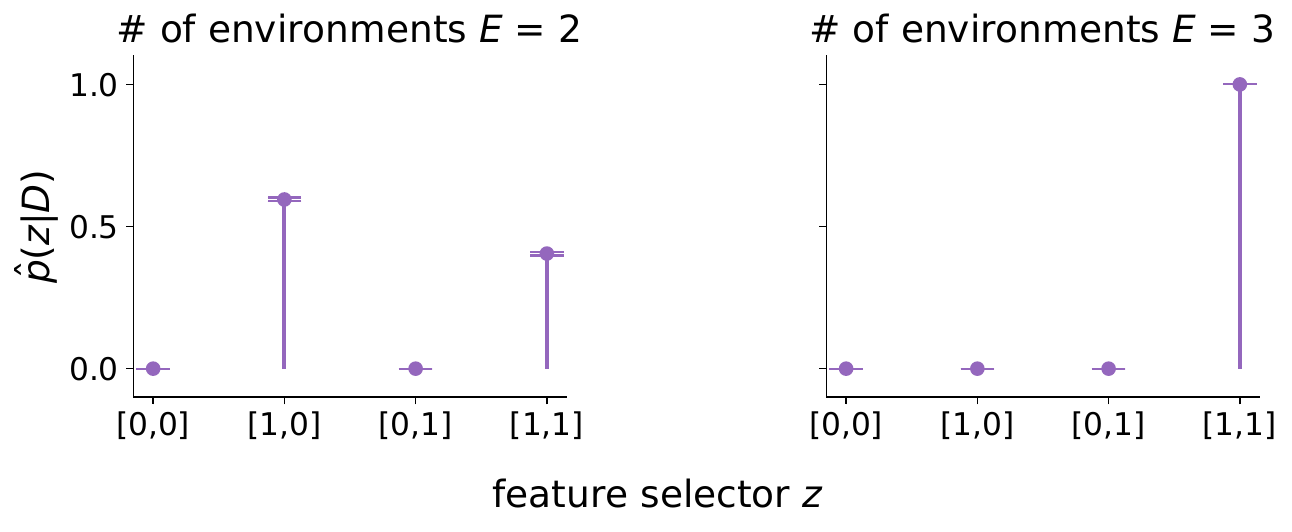}
    \caption{Example 2: posterior behavior under limited environments ($p=2$). With $E=2$ (left), the posterior $\hat p(z \mid \mathcal{D})$ is multi-modal over the two invariant selectors $z = [1,0]$ and $[1,1]$. For $E=3$ (right), only $z = [1,1]$ remains invariant, and the posterior concentrates there. Averaged over 1{,}000 simulations; 95\% confidence bands.}
    \label{fig:synthetic-uq-example2}
\end{figure}

\begin{figure}[!t]
    \centering
    \includegraphics[width=0.75\linewidth]{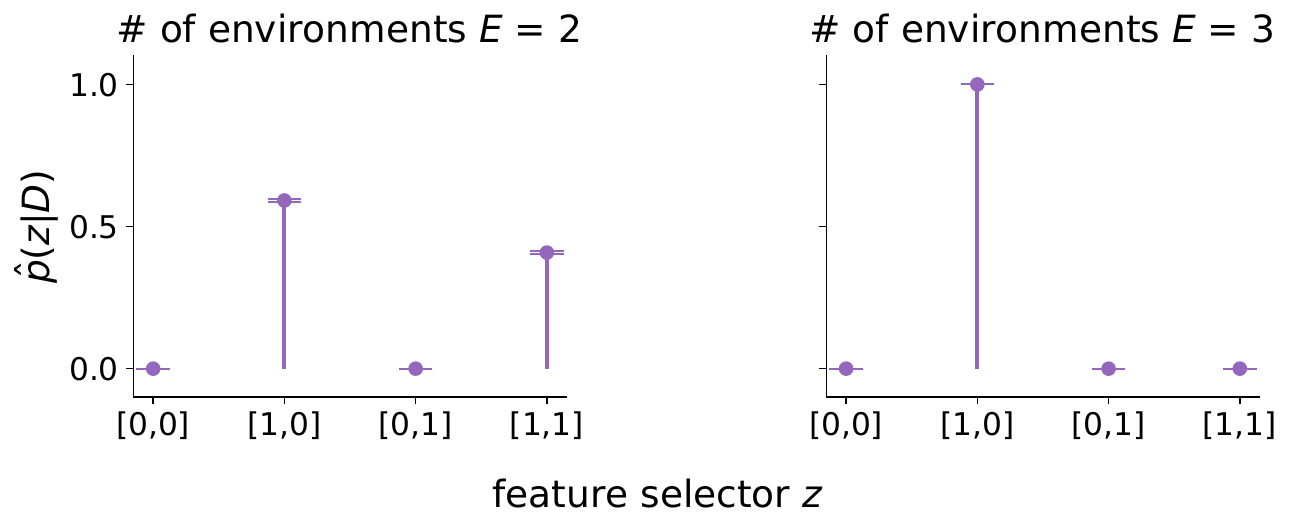}
    \caption{Example 3: posterior behavior under limited environments with $p=2$ features. For $E=2$ (left), the posterior $\hat p(z \mid \mathcal{D})$ is multi-modal over the two invariant selectors $[1,0]$ and $[1,1]$. With $E=3$ (right), only $[1,0]$ remains invariant, and the posterior concentrates at this value. Averaged over 1{,}000 simulations; 95\% confidence bands.}
    \label{fig:synthetic-uq-example3}
\end{figure}

\parhead{Example 3.} The joint distributions $\{p_e(x,y)\}_{e=1}^3$ are given by  
\begin{align*}
    p_1(x, y) &= p_1(\di{x}{1}) p(y \mid \di{x}{1}) p_1(\di{x}{2} \mid \di{x}{1}) \\
    p_2(x, y) &= p_2(\di{x}{1}) p(y \mid \di{x}{1}) p_2(\di{x}{2} \mid \di{x}{1}) \\
    p_3(x, y) &= p_3(\di{x}{1}) p(y \mid \di{x}{1}) p_3(\di{x}{2} \mid \di{x}{1}, y)
\end{align*}
where 
\begin{align*}
   & p_1(\di{x}{1}) = \mathcal{N}(\di{x}{1} \mid 0, 0.1^2), \, p_2(\di{x}{1}) = \mathcal{N}(\di{x}{1} \mid 2, 0.1^2),\, p_3(\di{x}{1}) = \mathcal{N}(\di{x}{1} \mid 5, 0.1^2) \\
   & p( y \mid \di{x}{1}) = \mathcal{N}(\di{x}{1} \mid \di{x}{1} + 0.5, 0.1^2) \\
   & p_1(\di{x}{2} \mid \di{x}{1}) = \mathcal{N}(\di{x}{2} \mid \di{x}{1} +0.1, 0.1^2), \, p_2(\di{x}{2}\mid \di{x}{1} ) = \mathcal{N}(\di{x}{2} \mid -2\di{x}{1} + 0.1, 0.1^2), \\
   &p_3(\di{x}{2} \mid \di{x}{1}, y) =  \mathcal{N}(\di{x}{2} \mid \di{x}{1} + y, 0.1^2). 
\end{align*}
The only set of invariant feature(s) across $\calE=\{1,2,3\}$ is $\di{x}{1}$, and the corresponding  invariant feature selector is $z^* = [1,0]$. 

In the first two environments $\calE'=\{1,2\}$, $\di{x}{1}, \di{x}{2}$ are also invariant features: 
\begin{align*}
  p_1(y \mid \di{x}{1}, \di{x}{2}) = p( y \mid \di{x}{1}) \\
  p_2(y \mid \di{x}{1}, \di{x}{2}) = p( y \mid \di{x}{1}) 
\end{align*}
since $\di{x}{2} \perp y \mid \di{x}{1}$ under both $p_1(x,y)$ and $p_2(x,y)$. The corresponding posterior distributions for the two scenarios are shown in \Cref{fig:synthetic-uq-example3}.

\subsection{Comparison to other methods}\label{app:subsec:synthetic-comparison}

For the experiments in \Cref{subsec:synthetic-comparison}, we include details on data simulation, method implementation, and additional results in \Cref{app:subsubsec:synthetic-comparison-simulation,app:subsubsec:synthetic-method-implementation,app:subsubsec:synthetic-comparison-additional-results}, respectively.

\subsubsection{Simulation details}\label{app:subsubsec:synthetic-comparison-simulation}
We follow the general data generative process described in the preamble of \Cref{app:sec:synthetic-data} with the following configuration (The reason to differentiate between various $p$ is to control the scale of parameter values to prevent the sampled $x,y$ values from exploding when $p$ increases.): 
\begin{enumerate}
    \item 
When $p=10$, 
\begin{enumerate}
    \item $p_{\text{act}} \sim  \uniform{\{0.6,0.7,0.8,0.9\}}$ 
    \item For $e=1$: (i) coefficient lower bound $lb=1$; (ii) coefficient upper bound $ub=2.1$, (iii) noise level lower bound $\sigma_{\min} = 0.1$; (iv) noise level upper bound $\sigma_{\max}=0.2$. 
    \item Separately for $e=2,\cdots, E$: for the intervened variable (i) mean of the absolute value of the intercept  $m_e\sim \uniform{[0,1]}$; (ii) coefficient lower bound $lb_e = (m_e+0.01)*2$; (iii) coefficient upper bound $ub_e = (m_e+0.5)*2$; (iv) variance multiplicative factor lower bound $\lambda_{e,\min} \sim \uniform{[0.1, 0.2]}$; (v) variance multiplicative factor relative range $\lambda_{e,\textrm{diff}} \sim  \uniform{[0.1, 0.5]}$; (vi)the probability of changing the coefficient $p_{\text{change}} \sim \uniform{[0.1,0.3]}$; and (vii) the fraction of features to be intervened on $\rho^{\text{int}}_e \sim \uniform{[0.5,1]}$
\end{enumerate}
\item 
When $p=450$,
\begin{enumerate}
    \item $p_{\text{act}} \sim  \uniform{\{0.1,0.15,0.2,0.25\}}$ 
    \item For $e=1$: same as $p=10$ case 
    \item Separately for $e=2,\cdots, E$: for the intervened variable (i) mean of the absolute value of the intercept  $m_e\sim \uniform{[0,0.4]}$; (ii) coefficient lower bound $lb_e= m_e+0.01$; (iii) coefficient upper bound $ub_e = (m_e+0.01)*1.5$; other simulation hyperparameters are chosen in the same way as $p=10$ case
\end{enumerate}
\end{enumerate}
We set the maximum number of the true invariant features $p^*_{\max} = 5$ when $p=10$, and  $p^*_{\max}=10$ otherwise. For all cases of $p$ we set the minimum number of the true invariant features $p_{\min}= 1$. 

\subsubsection{Method implementation details}
\label{app:subsubsec:synthetic-method-implementation}

\begin{enumerate}
    \item EILLS: We follow the implementation  at \url{https://github.com/wmyw96/EILLS} which accompanies the original paper by \citet{fan2023environment}. 

    \item ICP: We use the implementation at \url{https://github.com/juangamella/icp}. 
    
    \item Hidden-ICP: We implement a PyTorch version of the original R package available at \url{https://cran.r-project.org/package=InvariantCausalPrediction}, which accompanies the original paper by \citet{rothenhausler2019causal}. 

    \item BIP-VI: 
    The prior is uniform over $\{z \in \{0,1\}^p: \|z\|_0 \leq p_{\max}\}$. In low dimensions $(p=10)$, $p_{\max}=10$ leading to an uninformative prior; In high dimensions ($p=450$), $p_{\max}=20$ encouraging sparsity. 

    We run BIP-VI with two initialization strategies for the variational parameters and choose the one with better training-set ELBO. Consider each variational parameter $\phi_i$ initialized as $\log \frac{\sigma_{0i}}{1 - \sigma_{0i}}$ for $i 
    = 1, \dots, p$. The two initializations are the following: 
    (i) uninformative initialization: we set $\sigma_{0i} = \min(p_{\max} / p \times 0.9,\ 0.4)$ for all $i$ to encourage sparsity without strong prior belief; and 
    (ii) informative initialization: based on the screening step described in the main text, we assign $\sigma_{0i} = 0.5$ for the 10 screened indices and $\sigma_{0i} = 0.1$ for the remaining ones.

     As discussed in \Cref{app:subsec:vi-implementation}, when the sampled $z$ is outside of prior support, we set the corresponding objective value in \Cref{eq:vi-stochastic-objective} to $-1$. We always compute analytical gradients of KL, and stochastic gradients of the likelihood component in the ELBO.   
    
    In \Cref{alg:viml}, we use \( M = 20 \) gradient samples per iteration and optimize using SGD. For learning rate scheduling, we employ a cyclical learning rate strategy with a basic triangular cycle without amplitude scaling. Specifically, we set a base learning rate of 1.0, a maximum learning rate of 10, and a cycle length of 1000 iterations, with 500 steps in the increasing half. This scheduler can be implemented using the \texttt{PyTorch} package \citep{paszke2019pytorch}:
    \begin{quote}
    \texttt{torch.optim.lr\_scheduler.CyclicLR(optimizer, base\_lr=0.5, max\_lr=10, step\_size\_up=500, mode=`triangular')}
    \end{quote}

We run BIP-VI for a maximum of 2{,}000 iterations and select the variational parameters corresponding to the highest training-set ELBO.
\end{enumerate}

\subsubsection{Additional results}\label{app:subsubsec:synthetic-comparison-additional-results}

\begin{figure}[!t]
    \centering
    \includegraphics[scale=0.6]{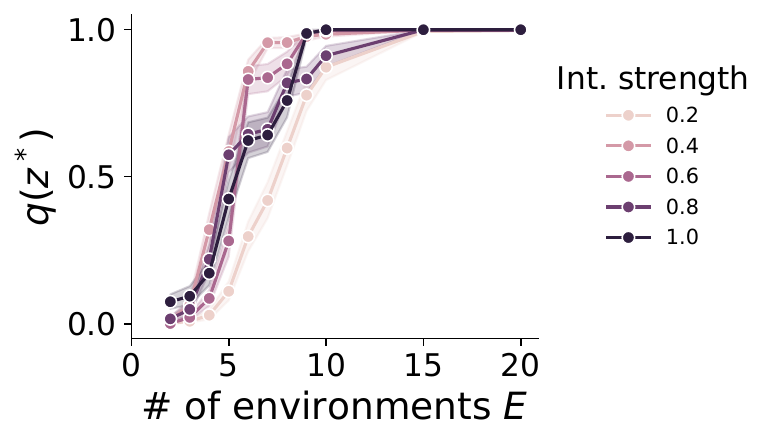}
    \caption{Synthetic study with $p=10$ features. Variational posterior at $z^*$ v.s. number of environments $E$ under different fraction of intervention features. Similar to \Cref{fig:synthetic-sweep-p10} (B), we observe that the variational posterior also concentrates at $z^*$ with increasing number of environments, and the convergence is faster when the fraction of intervention features is higher. The convergence of the variational posterior value at $z^*$ is slower than that under the exact posterior shown in \Cref{fig:synthetic-sweep-p10} (B).}
    \label{app:subfig:p10-posterior-vi}
\end{figure}

\Cref{app:subfig:p10-posterior-vi} displays the variational posterior probability at the true $z^*$ by BIP-VI  for the $p=10$ setting. Similar to the exact posterior results in \Cref{fig:synthetic-sweep-p10} (B), the variational posterior concentrate around $z^*$ as the number of environments $E$ increases, with faster convergence at under stronger interventions (which induce greated heterogeneity across environments).  However, the convergence of the variational posterior probability at $z^*$ is slower than  that of the exact posterior shown in \Cref{fig:synthetic-sweep-p10} (B).

%% file: appendix/app_gene_perturb_exp.tex
\section{Gene perturbation study}\label{app:sec:gene-study}
We include additional details for the experiments in \Cref{sec:gene-study}.

\parhead{Evaluation details.}
We evaluate each method based on whether its predicted invariant feature gene has a significant effect on the target gene. Specifically, we say the deletion of gene $j$ to have a significant effect on the target gene $i$, when (i) the deletion is successful, i.e. the expression level of gene $j$ is above the 99\% quantile or below the 1\% quantile of its observational level, and (ii) the expression level of the target gene $i$ changes significantly, which is above the 99\% quantile or below the 1\% quantile of its observational level.

\parhead{Hyperparameter settings in \Cref{fig:gene-results2}.}
In \Cref{fig:gene-results2}, each invariant inference method -- EILLS-screen, ICP-screen, Hidden-ICP-screen, BIP-screen, and BIP-VI -- is evaluated across multiple hyperparameter settings, represented by dots of varying sizes. Larger dots indicate more \textit{conservative} settings that predict fewer invariant feature genes.

\begin{itemize}
    \item EILLS-screen: Higher invariance regularization parameter $\gamma$ corresponds to more conservative behavior. We vary $\gamma \in \{1, 25, 50, 75, 100\}$.
    \item ICP-screen: Lower significance level $\alpha$ results in more conservative predictions. We vary $\alpha \in \{0.001, 0.005, 0.01, 0.05, 0.1\}$.
    \item Hidden-ICP: In contrast, higher $\alpha$ leads to more conservative behavior. We vary $\alpha \in \{0.001, 0.005, 0.01, 0.05, 0.1\}$.
    \item BIP-screen and BIP-VI: Higher posterior thresholds $t$ lead to more conservative results. We vary $t\in \{0.5, 0.6, 0.7, 0.8, 0.9\}$.
\end{itemize}

\parhead{Implementation details.}
The implementations of EILLS, ICP, and Hidden-ICP follow the same setup as described in  \Cref{app:subsubsec:synthetic-method-implementation}.

For BIP-VI, we initialize the variational parameters $\phi_i$ to  $\log \frac{\sigma_0}{1-\sigma_0}$ for $i=1,\cdots, p$, where the initial feature selection probability $\sigma_0:=0.02$. We choose the prior to be uniform over $\{z \in \{0,1\}^p: \|z\|_0 \leq p_{\max}\}$ with $p_{\max}=200$. These choices are determined by coarsely searching over reasonable values guided by the optimization dynamics. A good starting point is to ensure that $p_{\max}$ is smaller than the per-environment sample size $n$ so that the MLEs for local conditional models are well-defined. Additionally, it is desirable that $p \sigma_0 < p_{max}$ so that the initial expected selected features is within the prior support.

As discussed in \Cref{app:subsec:vi-implementation}, when the sampled $z$ is outside of prior support, we set the corresponding objective value in \Cref{eq:vi-stochastic-objective} to $-1$. With probability of $0.5$ we compute the analytical gradients of the KL regularization term in the ELBO; otherwise we compute the stochastic gradients.

In \Cref{alg:viml}, we use \( M = 10 \) gradient samples per iteration and optimize using SGD. For learning rate scheduling, we employ a cyclical learning rate strategy with a basic triangular cycle that scales initial amplitude by half each cycle. Specifically, we set a base learning rate of 0.5, a maximum learning rate of 10, and a cycle length of 1000 iterations, with 500 steps in the increasing half. This scheduler can be implemented using the \texttt{PyTorch} package:
    \begin{quote}
    \texttt{torch.optim.lr\_scheduler.CyclicLR(optimizer, base\_lr=0.5, max\_lr=10, step\_size\_up=500, mode=`triangular2')}
    \end{quote}

We run BIP-VI for a maximum of 10{,}000 iterations and select the variational parameters corresponding to the highest training-set ELBO.


%% file: appendix/app_time_profile.tex
\section{Computational cost}\label{app:sec:compute}
We profile the wall-clock cost of BIP-VI, BIP exact inference (potentially with feature prescreening), and all baselines on both the synthetic
benchmark and the gene-perturbation experiment.

\parhead{Setup.}
All methods are timed on the same hardware (CPU, $8$ threads, $32$\,GB RAM). We time the model fit only,
excluding data preprocessing and result summarization. For methods that require feature
pre-screening (\Cref{subsec:comparison-methods}), the reported time includes the pre-screen, so
each cost reflects the full pipeline; the pre-screen is a single LassoCV computation shared
across the screened baselines (ICP, BIP-exact, EILLS, and Hidden-ICP).

BIP-VI and the marginal-regression baseline use no pre-screen and operate on the full
feature pool. The BIP-VI prior instead caps the candidate invariant set at a maximum size
$\|z\|_0 \le p_{\max}$. This cap is often motivated by a statistical constraint: with limited
per-environment sample size, the conditional least-squares systems become underdetermined once
the active set approaches $n$, so the search must be restricted to sufficiently sparse subsets;
when available, domain knowledge about the plausible number of invariant predictors can also be
encoded through $p_{\max}$. Because BIP-VI's per-iteration cost grows only polynomially in
$p_{\max}$ (\Cref{subsec:vi}), this cap---rather than the full feature dimension $p$---is the main
driver of its computational cost. We set $p_{\max}=10$ for the low-dimensional synthetic
experiments ($p=10$, so the cap is inactive), $p_{\max}=20$ for the high-dimensional synthetic
setting ($p=450$), and $p_{\max}=200$ for the gene experiment ($p=6{,}169$).

The two experiments use different BIP-VI optimization budgets: $2{,}000$ iterations with $M=20$ for the synthetic runs, versus $10{,}000$
iterations with $M=10$ for the gene runs. They serve complementary purposes, and their absolute timings are not directly comparable: the
synthetic benchmark isolates how cost scales with $n$, $E$, and $p$ under a controlled setting,
whereas the gene experiment targets a real high-dimensional one.

\parhead{Synthetic benchmark.}
We vary $n$, $E$, and $p$ in turn around the anchor configuration $(p{=}10,\, n{=}200,\, E{=}5)$,
holding the other two fixed and sweeping the intervention strength at every point. \Cref{tab:compute-synthetic} reports the median
total wall-clock over $4$ intervention strengths $\times$ $3$ data-generating seeds
($12$ runs) per configuration.

\begin{table}[t]
\centering
\small
\begin{tabular}{l rrrr rrr r}
\toprule
& \multicolumn{4}{c}{$n$ \;($p{=}10,\,E{=}5$)} & \multicolumn{3}{c}{$E$ \;($p{=}10,\,n{=}200$)} & {$p$} \\
\cmidrule(lr){2-5}\cmidrule(lr){6-8}\cmidrule(lr){9-9}
Method & $50$ & $200$ & $500$ & $2000$ & $2$ & $10$ & $20$ & $450$ \\
\midrule
BIP-VI & $8.54$ & $8.77$ & $9.56$ & $12.02$ & $8.45$ & $8.86$ & $9.38$ & $11.43$ \\
ICP         & $6.28$  & $6.33$  & $6.39$  & $6.68$  & $2.60$  & $12.62$ & $25.52$ & $6.35$  \\
BIP-exact   & $0.27$  & $0.30$  & $0.36$  & $0.65$  & $0.16$  & $0.53$  & $0.99$  & $0.34$  \\
EILLS       & $0.12$  & $0.12$  & $0.12$  & $0.12$  & $0.06$  & $0.21$  & $0.38$  & $0.15$  \\
Hidden-ICP  & $0.01$  & $0.04$  & $0.09$  & $0.36$  & $0.004$ & $0.15$  & $0.57$  & $0.08$  \\
\bottomrule
\end{tabular}
\caption{Median wall-clock (seconds) on the synthetic benchmark, varying $n$, $E$, and $p$ in
turn around the anchor $(p{=}10,\,n{=}200,\,E{=}5)$; each cell is the median over $12$ runs
($4$ intervention strengths $\times$ $3$ seeds), with small spread throughout (interquartile
range $\le 0.95$\,s for BIP-VI, $\le 0.42$\,s for ICP, and $\le 0.02$\,s for the others). At
$p{=}450$ every method except BIP-VI operates on $10$ LassoCV-screened features, with the
screening time ($\approx 0.04$\,s) included in the reported total. BIP-VI scales only mildly
with all three factors, whereas ICP grows sharply in $E$; the screened baselines run on $10$
features at both $p{=}10$ and $p{=}450$ and so are essentially flat in $p$. The oracle and
plain-regression references ($\approx 10^{-3}$\,s) are omitted.}
\label{tab:compute-synthetic}
\end{table}

From \Cref{tab:compute-synthetic}, BIP-VI's runtime grows roughly linearly in $n$ and only
weakly in $E$ and $p$, consistent with its $O(T \cdot M \cdot c(\calD, p_{\max}))$ complexity.
Among the baselines, ICP is the most sensitive to $E$, rising about tenfold from $E{=}2$ to
$E{=}20$, whereas BIP-exact, EILLS, and Hidden-ICP remain at or below roughly one second across
all configurations. ICP's sharp growth in $E$ reflects its testing procedure: in the original implementation, which
we adopt unchanged, it loops over environments to run an accept/reject invariance test for each
candidate subset, whereas BIP-VI vectorizes the per-environment conditional fits and so depends
only weakly on $E$.

\parhead{Gene-perturbation experiment.}
We profile every method on the $10$ target genes of \Cref{tab:gene-study-consistency}. For each
target, we predict its expression from all other
genes---a single regression problem with $E=2$ environments (observational vs.\ pooled
interventional) and one cross-validation fold. \Cref{tab:compute-gene} reports the mean $\pm$
standard error over $30$ runs ($10$ target genes $\times$ $3$ seeds).

\begin{table}[t]
\centering
\small
\begin{tabular}{ll r}
\toprule
Method & Candidate feature sets & Wall-clock (s) \\
\midrule
BIP-VI      & full pool ($\le p_{\max}{=}200$) & $348.7 \pm 17.2$ \\
ICP         & top-$10$ screen                  & $11.0 \pm 1.2$   \\
BIP-exact   & top-$10$ screen                  & $9.2 \pm 1.2$    \\
EILLS       & top-$10$ screen                  & $9.0 \pm 1.2$    \\
Hidden-ICP  & top-$10$ screen                  & $8.7 \pm 1.2$    \\
Marginal    & full pool (marginal only)        & $0.13 \pm 0.01$  \\
\bottomrule
\end{tabular}
\caption{Wall-clock (seconds) on the yeast gene-perturbation experiment, as mean $\pm$ standard
error over $30$ runs ($10$ target genes from \Cref{tab:gene-study-consistency} $\times$ $3$
seeds; $E=2$, one cross-validation fold). BIP-VI runs on the full $6{,}169$-gene pool
($p_{\max}=200$, $10{,}000$ iterations, $M=10$); the screened baselines act on the top-$10$
LassoCV screen, whose shared $\approx 8.7$\,s cost---included in the totals---dominates their
runtime.}
\label{tab:compute-gene}
\end{table}

From \Cref{tab:compute-gene}, BIP-VI ($349$\,s per target gene) is the most expensive, while
the screened baselines cluster at $9$--$11$\,s---dominated by the shared $\approx 8.7$\,s
pre-screen rather than their own fits---and marginal regression is cheapest at $0.13$\,s. These
gaps track how large a candidate space each method searches: BIP-VI optimizes over all subsets
of up to $p_{\max}=200$ genes drawn from the full pool, the screened baselines enumerate only
within the $10$ pre-screened genes, and marginal regression scores each gene by its marginal
correlation.

\parhead{Summary.}
Taken together, BIP-VI's cost is governed by its optimization budget---the number of iterations
$T$ and Monte Carlo samples $M$---rather than by the size of the candidate-subset space. Each
gradient sample only fits and evaluates conditional models on a subset of size at most
$p_{\max} \le p$, at a cost $c(\calD, p_{\max})$ that is polynomial in $p_{\max}$ (a linear-Gaussian
least-squares fit; \Cref{subsec:vi}). The total cost $O(T \cdot M \cdot c(\calD, p_{\max}))$ is
therefore polynomial in $p_{\max}$, never growing with the number of candidate subsets.

This is the central distinction from the enumeration-based baselines (ICP, BIP-exact, EILLS),
which stay tractable only by pre-screening to a small pool. Enumerating the subsets BIP-VI
considers is infeasible: at $p_{\max}{=}200$ over the full gene pool there are
$\sum_{k \le 200} \binom{6169}{k} \approx 10^{382}$ subsets, which even at ICP's measured rate
($\approx 2.3$\,s for $2^{10}=1{,}024$ subsets, about a millisecond each) would take on the order
of $10^{371}$ years---and conservatively so, since larger subsets cost more. BIP-VI instead
recasts the search as a continuous optimization over a variational distribution on subsets,
navigating the space with gradient information at a fixed, predictable cost.